\documentclass[12pt]{article}

\usepackage{amsmath}
\usepackage{amsfonts}
\usepackage{amsthm}
\usepackage{graphicx}
\usepackage{subfigure} 
\usepackage{natbib}
\usepackage{algorithm}
\usepackage{algorithmic}
\usepackage{hyperref}
\usepackage{float}
\usepackage{lipsum}
\usepackage{fullpage}

\newtheorem{defini}{Definition}
\newtheorem{lem}{Lemma}
\newtheorem{prop}{Proposition}
\newtheorem{them}{Theorem}
\newtheorem{assum}{Assumption}

\graphicspath{{eps/}}
\begin{document} 
	\newcommand{\argmax}{\mathop{\rm argmax}\limits}
\newcommand{\argmin}{\mathop{\rm argmin}\limits}

%Highlight
%\newcommand{\hilight}[1]{\colorbox{yellow}{#1}}
%Disable Highlight
\newcommand{\hilight}[1]{{#1}}
\newcommand{\highlight}[2][yellow]{\mathchoice%
	{\colorbox{#1}{$\displaystyle#2$}}%
	{\colorbox{#1}{$\textstyle#2$}}%
	{\colorbox{#1}{$\scriptstyle#2$}}%
	{\colorbox{#1}{$\scriptscriptstyle#2$}}}%

\newcommand{\unorm}[1]{\|#1\|}
\newcommand{\unorms}[1]{\unorm{#1}^2}
\newcommand{\calX}{{\mathcal{X}}}
\newcommand{\calY}{{\mathcal{Y}}}
\newcommand{\gonetwo}{{g_{12}}}
\newcommand{\gonetwobar}{{\bar{g}_{12}}}
\newcommand{\calB}{{\mathcal{B}}}
\newcommand{\calBitwo}{{\mathcal{B}_{i,2}}}
\newcommand{\calBione}{{\mathcal{B}_{i,1}}}
\newcommand{\boldtheta}{{\boldsymbol{\theta}}}
\newcommand{\bolddelta}{{\boldsymbol{\delta}}}
\newcommand{\boldthetaP}{{\boldsymbol{\theta}}^{(p)}}
\newcommand{\boldthetaQ}{{\boldsymbol{\theta}}^{(q)}}
\newcommand{\boldthetaPtop}{{\boldsymbol{\theta}}^{(p)\top}}
\newcommand{\factorp}{{\phi}^P}
\newcommand{\factorq}{{\phi}^Q}
\newcommand{\boldalpha}{{\boldsymbol{\alpha}}}
\newcommand{\boldHh}{{\widehat{\boldH}}}
\newcommand{\boldeta}{{\boldsymbol{\eta}}}
\newcommand{\boldH}{{\boldsymbol{H}}}
\newcommand{\boldA}{{\boldsymbol{A}}}
\newcommand{\boldS}{{\boldsymbol{S}}}
\newcommand{\boldK}{{\boldsymbol{K}}}
\newcommand{\boldJ}{{\boldsymbol{J}}}
\newcommand{\boldT}{{\boldsymbol{T}}}
\newcommand{\boldTheta}{{\boldsymbol{\Theta}}}
\newcommand{\boldf}{{\boldsymbol{f}}}
\newcommand{\boldu}{{\boldsymbol{u}}}
\newcommand{\bolda}{{\boldsymbol{a}}}
\newcommand{\boldm}{{\boldsymbol{m}}}
\newcommand{\boldone}{{\boldsymbol{1}}}
\newcommand{\boldxi}{{\boldsymbol{\xi}}}
\newcommand{\boldv}{{\boldsymbol{v}}}
\newcommand{\boldk}{{\boldsymbol{k}}}
\newcommand{\boldb}{{\boldsymbol{b}}}
\newcommand{\boldbeta}{{\boldsymbol{\beta}}}
\newcommand{\boldDelta}{{\boldsymbol{\Delta}}}
\newcommand{\nnu}{\nsample}
\newcommand{\nsample}{n}
\newcommand{\subsetr}{\boldsymbol{r}}
\newcommand{\boldthetah}{{\widehat{\boldtheta}}}
\newcommand{\mathbbR}{\mathbb{R}}
\newcommand{\KL}{\mathrm{KL}}
\newcommand{\numparams}{n}
\newcommand{\boldhh}{{\widehat{\boldh}}}
\newcommand{\boldh}{{\boldsymbol{h}}}
\newcommand{\Hh}{{\widehat{H}}}
\newcommand{\boldxnu}{\boldY}
\newcommand{\boldx}{{\boldsymbol{x}}}
\newcommand{\boldxp}{{\boldsymbol{x}}_{p}}
\newcommand{\boldxq}{{\boldsymbol{x}}_{q}}
\newcommand{\boldz}{{\boldsymbol{z}}}
\newcommand{\boldg}{{\boldsymbol{g}}}
\newcommand{\boldw}{{\boldsymbol{w}}}
\newcommand{\boldr}{{\boldsymbol{r}}}
\newcommand{\boldQ}{{\boldsymbol{Q}}}
\newcommand{\boldF}{{\boldsymbol{F}}}
\newcommand{\boldzero}{{\boldsymbol{0}}}
\newcommand{\thetahat}{{\hat{\boldsymbol{\theta}}}}
\newcommand{\thetaShat}{{\hat{\boldsymbol{\theta}}_S}}
\newcommand{\thetaSchat}{{\hat{\boldsymbol{\theta}}_{S^c}}}
\newcommand{\zhat}{{\hat{\boldsymbol{z}}}}
\newcommand{\zSchat}{{\hat{\boldsymbol{z}}_{S^c}}}
\newcommand{\zShat}{{\hat{\boldsymbol{z}}_{S}}}
\newcommand{\nde}{\nsample'}
\newcommand{\boldxde}{\boldY'}
\newcommand{\boldX}{{\boldsymbol{X}}}
\newcommand{\boldY}{{\boldsymbol{Y}}}
\newcommand{\boldy}{{\boldsymbol{y}}}
\newcommand{\boldt}{{\boldsymbol{t}}}
\newcommand{\boldc}{{\boldsymbol{c}}}
\newcommand{\boldYnu}{{\boldsymbol{Y}}}
\newcommand{\boldYde}{{\boldsymbol{Y}}}
\newcommand{\boldpsi}{{\boldsymbol{\psi}}}
\newcommand{\hh}{{\widehat{h}}}
\newcommand{\boldI}{{\boldsymbol{I}}}
\newcommand{\PE}{{\widehat{PE}}}
\newcommand{\ratioh}{\widehat{\ratiosymbol}}
\newcommand{\ratiosymbol}{r}
\newcommand{\ratiomodel}{g}
\newcommand{\thetah}{{\widehat{\theta}}}
\newcommand{\mathbbE}{\mathbb{E}}
\newcommand{\pnu}{p_\mathrm{te}}
\newcommand{\pde}{p_\mathrm{rf}}
\newcommand{\refsection}{\boldS_\mathrm{rf}}
\newcommand{\tesection}{\boldS_\mathrm{te}}
\newcommand{\refY}{\boldY_\mathrm{rf}}
\newcommand{\teY}{\boldY_\mathrm{te}}
\newcommand{\nseg}{n}
\newcommand{\distP}{P}
\newcommand{\distQ}{Q}
\newcommand{\iid}{\stackrel{\mathrm{i.i.d.}}{\sim}}
\newcommand{\dx}{\mathrm{d}\boldx}
\newcommand{\dy}{\mathrm{d}\boldy}

\newcommand{\gxeta}{g(\boldx;\boldeta)}
\newcommand{\Zeta}{Z(\boldeta)}
\newcommand{\Zetahat}{\hat{Z}(\boldeta)}

\def\ratio{r}
\def\relratio{{\ratio}_{\alpha}}

\def\ci{\perp\!\!\!\perp} % from Wikipedia
\newcommand\independent{\protect\mathpalette{\protect\independenT}{\perp}} % symbols-a4, p.106
\def\independenT#1#2{\mathrel{\rlap{$#1#2$}\mkern2mu{#1#2}}} 
\newcommand*\xor{\mathbin{\oplus}}

\newcommand{\vertiii}[1]{{\left\vert\kern-0.25ex\left\vert\kern-0.25ex\left\vert #1 
		\right\vert\kern-0.25ex\right\vert\kern-0.25ex\right\vert}}
	\title{Structure Learning of Partitioned Markov Networks}\date{}
	\author{Song Liu\\
		\url{liu@ism.ac.jp}\\
		The Institute of Statistical Mathematics, Japan\\[2mm]
		Taiji Suzuki\\
		\url{suzuki.t.ct@m.titech.ac.jp}\\
		Tokyo Institute of Technology, Japan\\
		[2mm]
		Masashi Sugiyama\\
		\url{sugi@k.u-tokyo.ac.jp}\\
		University of Tokyo, Japan\\[2mm]
		Kenji Fukumizu\\
		\url{fukumizu@ism.ac.jp}\\
		The Institute of Statistical Mathematics, Japan\\
	}
	\maketitle
\begin{abstract}
We learn the structure of a Markov Network between two groups of random variables from joint observations. 
Since modelling and learning the full MN structure may be hard, learning the links between two groups directly may be a preferable option. 
We introduce a novel concept called 
the \emph{partitioned ratio} whose factorization directly associates with the Markovian properties of random variables across two groups.
A simple one-shot convex optimization procedure is proposed for learning the \emph{sparse} factorizations of the partitioned ratio and it is theoretically guaranteed to recover the correct inter-group structure under mild  conditions.
The performance of the proposed method is experimentally compared with the state of the art MN structure learning methods using ROC curves. Real applications on analyzing bipartisanship in US congress and pairwise DNA/time-series alignments are also reported.
\end{abstract} 

\section{Introduction}
\label{sec.intro}
An undirected graphical model, or a Markov Network (MN) \citep{PGM_Koller,GM_Wainwright} has a wide range of applications in real world, such as natural language processing, computer vision, and computational biology. The structure of MN, which encodes the interactions among random variables, is one of the key interests of MN learning tasks. However, on a high-dimensional dataset, learning the full MN structure can be cumbersum since we may not have enough knowledge to model the entire MN, or our application only concerns a specific portion of the MN structure.

Rather than considering the full MN structure over the complete set of random variables, we focus on learning a portion of the MN structure that links \emph{two groups of random variables}, namely the Partitioned Markov Network (PMN).   
\begin{figure}
	\centering
	\includegraphics[width=.6\textwidth]{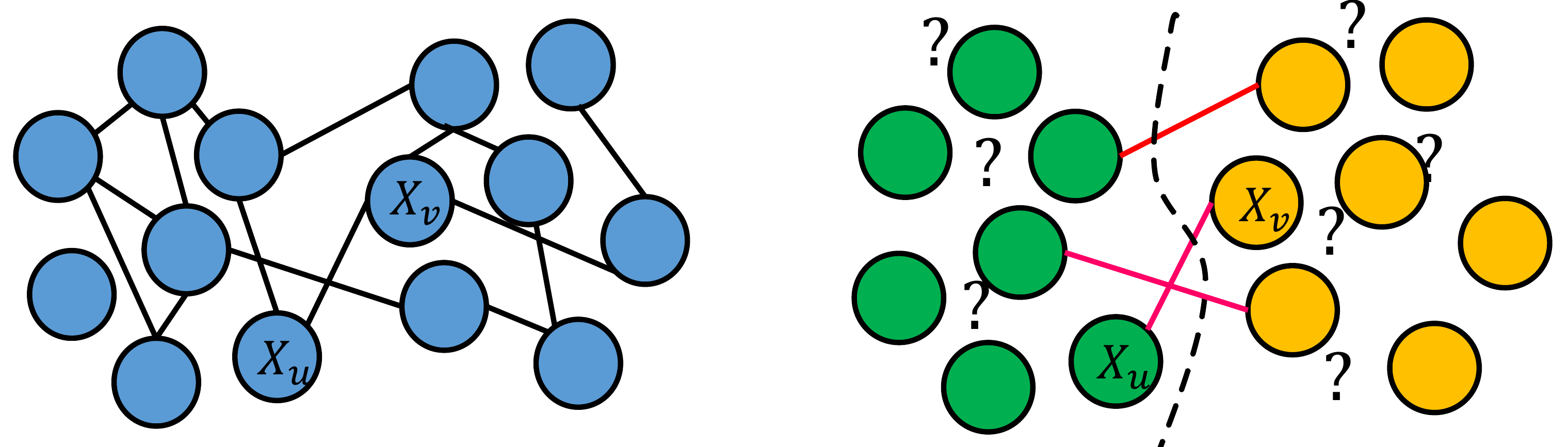}
	\caption{An illustration of a full MN (left) and PMN (right). Full MN models all the connections among random variables, while PMN only models the interactions between groups (red edges) and does not care connections within groups.}
	\label{fig.illus}
	\vspace*{-5mm}
\end{figure}
PMN is suitable for describing the ``inter-group relations''. For example, 
%the bipartisan collaborations among US congressmen can be described via PMN, as politicians are naturally partitioned into two parties. 
politicians in US Congress are naturally grouped into two parties (Democrats and Republicans). Learning a PMN on congresspersons via their voting records will reveal bipartisan collaborations among them. 
A full gene network may have complicated structure. However if genes can be clustered into a few homologous groups, PMN can help us understand how genes in different functioning groups interact with each other. An illustration of a full MN and a PMN is shown in Figure \ref{fig.illus}.

%Genes are often grouped together by their functions. Though genes within each group have complicated connections, we only interested in 

%or regulations between two sets of genes via gene expressions. 

Since a PMN can be regarded as a ``sub-structure'' of a full MN, a naive approach may be learning a full MN over the complete set of random variables and figuring out its PMN.
In fact, the machine learning community has seen huge progresses on learning the \emph{sparse} structures of MNs, thanks to the pioneer works on sparsity inducing norms \citep{tibshirani1996lasso,zhao2006model,WainwrightL1Sharp}.
%According to the Hammersley-Clifford theorem \citep{HC_them}, the joint distribution of an MN factorizes over smaller subgraph structures, cliques based on conditional independence. By learning a sparse factorization of the joint distribution, we may recover the corresponding structure of a full MN.

A majority of the previous works fall into the category of the regularized maximum likelihood approach which maximizes the likelihood function of a probabilistic model under sparsity constrains. Graphical lasso \citep{Friedman_GLasso,Banerjee_Model_Selection} considers a joint Gaussian model parameterized by the inverse covariance matrix, where zero elements indicate the conditional independence among random variables, while others have developed useful variations of graphical lasso in order to loosen the Gaussianity assumed on data \citep{nonparanormal,Poling-loh2012}. SKEPTIC \citep{liu2012high} is a semi-parametric approach that replaces the covariance matrix with the correlation matrix, such as Kendall's Tau in MN learning.

The latest advances along this line of research has been made by considering a node-wise conditional probabilistic model. Instead of learning all the structures in one shot, such a method focuses on learning the neighborhood structure of a single random variable at a time. Maximizing the conditional likelihood leads to simple logistic regression (in the case of the Ising model) \citep{Ravikumar_2010} or linear regression (in the case of the Gaussian model) \citep{MeinshausenBuhlmann}. 

Unfortunately, the maximum (conditional) likelihood method can be difficult to compute for general non-Gaussian graphical models, since computing the normalization term is in general intractable. Though one may use sampling such as Monte-carlo methods \citep{RobertMCStat2005} to approximate the normalization term, there is no universal guideline telling how to choose sampling parameters so that the approximation error is minimized. 

A more severe problem is that sparsity  approaches may have difficulties when learning a dense MN. Specifically, the samples size required for a successful structure recovery grows quadratically with the number of  connected neighbors  \citep{RaskuL1GMRF,Ravikumar_2010}. However, it is quite reasonable to assume that in some applications, one node may have many neighbors within its own group while connections to the other group are sparse: a congressperson is very well connected to other members inside his/her party but has only a few links with the opposition party. Genes in a homologous group may have dense structure but they only interact with another group of genes via a few ties. 

Is there a way to \emph{directly} obtain the PMN structure? Neither maximizing a joint nor conditional likelihood take the  ``partition information'' into account and interactions are modelled \emph{globally}. However PMN encodes only the \emph{local} conditional independence between groups, and the requirement for obtaining a good estimator should be much milder.

%Intuitively, the probability density/mass function itself completely characterizes a distribution, i.e., knowing density/mass means knowing everything about a distribution. Therefore, a density/mass function can be very hard to learn. On the other hand, the conditional independence is a very fundamental property among random variables. Learning a quantity that is sufficient to tell the structure of PMN should be easier than modelling the entire distribution. 

The above intuition leads us to a novel concept of the Partitioned Ratio (PR). Given a set of partitioned random variables $X = (X1, X2)$, PR is the ratio between the joint probability $P(X)$ and the product between its marginals $P(X1)P(X2)$, i.e. $\frac{P(X)}{P(X1)P(X2)}$. 
%In this paper, we argue that this quantity factorizes over the subgraph structure of an PMN. 
In the same way that the joint distribution can be decomposed into clique potentials of MN, we prove PR also factorizes over subgraph structures called \emph{passages}, which indicate the connectivity between two groups of random variables $X1$ and $X2$ in a PMN.

Conventionally, PR is a measure of the independence between two sets of random variables. In this paper, we show that the factorization of this quantity indicates the linkage between two groups of random variables, which is a natural extension of the regular usage of PR. 

Most importantly, we show the sparse factorization of this quantity may be learned via a one shot convex optimization procedure, which can be solved efficiently even for the general, non-Gaussian distributions. The correct recovery of sparse passage structure is theoretically guaranteed under the assumption that the sample size increases with the number of passages which is not related to the structure density of the entire MN. 

This paper is organized as follows. In Section \ref{sec.background}, we review the Hammersley and Clifford theorem (Section \ref{sec.HC}) and define some notations as preliminaries (Section \ref{sec.preliminaries}). The factorization theorems of PMN are introduced in Section \ref{sec.factorization} with a few simplifications. We give an estimator to obtain the sparse factorization of PR in Section \ref{sec.algorithm} and prove its recovered structure is consistent in Section \ref{sec.consistency}. Finally, experimental results on both artificial and real-world datasets are reported in Section \ref{sec.exp}. 

\section{Background and Preliminaries}
\label{sec.background}
%It is well known that a probability distribution $P(X)$ can be factorized according to a graph $G$ if and only if $X$ satisfies Markovian properties indicated by such a graph. 
In this section, we review the factorization theorems of MN.
We limit our discussions on strictly positive distributions from now on. A graph is always assumed to be finite, simple, and undirected.

\subsection{Background and Motivation}
\label{sec.HC}
\begin{defini}[MN]
	\label{def.MRF}
	For a joint probability $P(X)$ of random variables $X = \{X_1, X_2, \dots, X_m\}$, if for all $i$, $P(X_i|\backslash X_{ i}) = P(X_i | X_{N(i)})$, where $X_{N(i)}$ is the neighbors of node $X_i$ in graph $G$, then $P$ is an MN with respect to $G$.
\end{defini}

\begin{defini}[Gibbs Distribution]
	\label{def.Fact}
	For a joint distribution $P$ on a set of random variables $X$, if the joint density can be factorized as
	$$
	P(X) = \frac{1}{Z}\prod_{C \in \mathbf{C}(G)} \phi_{C}(X_{C}),
	$$
	where $Z$ is the normalization term, $\mathbf{C}(G)$ is the set of complete subgraphs of $G$ and each factor $\phi_{C}$ is defined only on a subset of random variables $X_C$, then $P$ is called a Gibbs distribution that factorizes over $G$.
\end{defini}
\begin{them}[See e.g., \citet{HC_them}]
	\label{thm.HC}
	If $P$ is an MN with respect to $G$ (Definition \ref{def.MRF}), then
	$P$ is a Gibbs distribution that factorizes over $G$ (Definition \ref{def.Fact})
\end{them}

\begin{them}[See e.g., \citet{PGM_Koller}]
	\label{thm.gibbs}
	If $P$ is a Gibbs distribution that factorizes over $G$ then $P$ is an MN with respect to $G$.
\end{them}
Theorems \ref{thm.HC} and \ref{thm.gibbs} are the keystones of many MN structure learning methods. It states, by learning a sparse factorization of a joint distribution, we are able to spot the structure of a graphical model. However, learning a joint distribution has never been an easy task due to the normalization issue and if the task is to learn a PMN that only concerns conditional independence across two groups, such an approach seems to ``solve a more general task as an intermediate step''\citep{Vapnik1998}.

Does there exist an alternative to the joint distribution, whose factorization relates to the structure of PMN? Ideally, such factorization should be efficiently estimated from samples with a tractable normalization term and the estimation procedure should provide good statistical guarantees.
%Interestingly, a very promising candidate has been long-existed in probability theories.

In the rest of the paper, we show PR has the desired properties to indicate the structure of a PMN: It is factorized over the structure of a PMN (Section \ref{sec.factorization}) and easy to estimate from joint samples (Section \ref{sec.algorithm}) with good statistical properties (Section \ref{sec.consistency}). 
%Before introducing the factorization theorem for PMN, we need a few rigorous definitions.
\subsection{Definitions}
\label{sec.preliminaries}
\paragraph{Notations.} Sets are denoted by upper-case letters, e.g., $A,B$. An upper-case with a lower-case subscript $A_i$ means the $i$-th element in $A$. Set operator $A\backslash B$ means excluding set $B$ from set $A$. $\backslash B$ means the whole set excluding the set $B$.  $A = (A1, A2)$ is a partition of set $A$ and an upper-case followed by an integer number, e.g. $A1, A2$ means groups divided by such a partition. Given a graph $L=\langle N,E \rangle$ and a subgraph $K\subseteq L$, $N_K$ or $E_K$ denotes the subset of $N$ or $E$ whose elements are indexed topologically by $K$. Upper-case with bold font, e.g. $\mathbf{K}$,  is a set of sets.

\paragraph{PMN and Gibbs Partitioned Ratio.}
\begin{figure}[t]
	\centering
	\includegraphics[width=.6\textwidth]{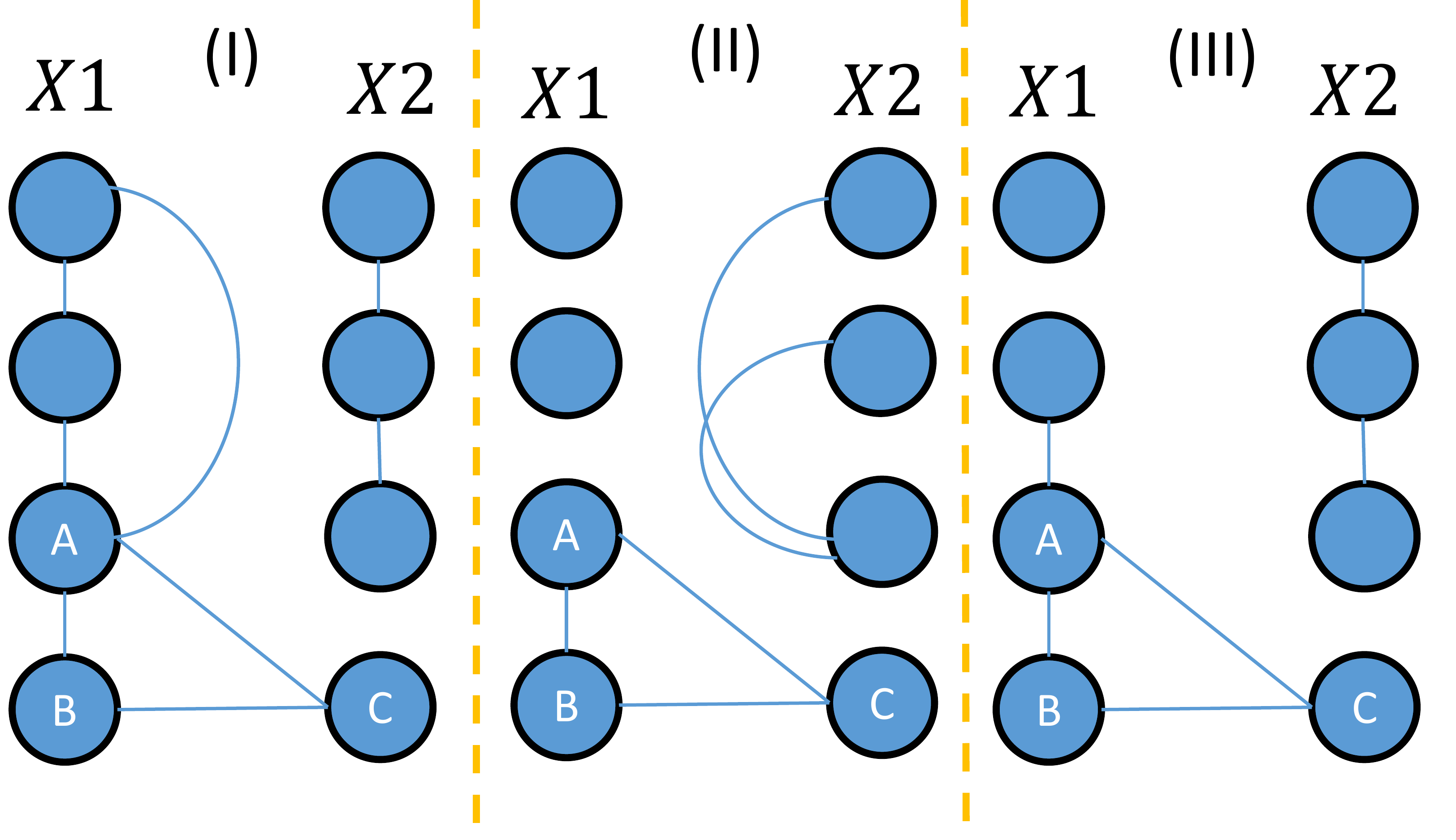}
	\caption{If (I) is an MN over $X$, then (I), (II), (III) are all PMNs over $X$.
%	since $\forall i, X1 \cup X_{N(i)} \backslash X_i$ are identical for all three graphs. 
	If (I) is a PMN over $X$, (I), (II), (III) are not necessarily the MN over $X$ (but still PMNs over $X$).}
	\label{fig.MNvsPMN}
\end{figure}
\begin{figure}[t]
	\centering	\includegraphics[width=.6\textwidth]{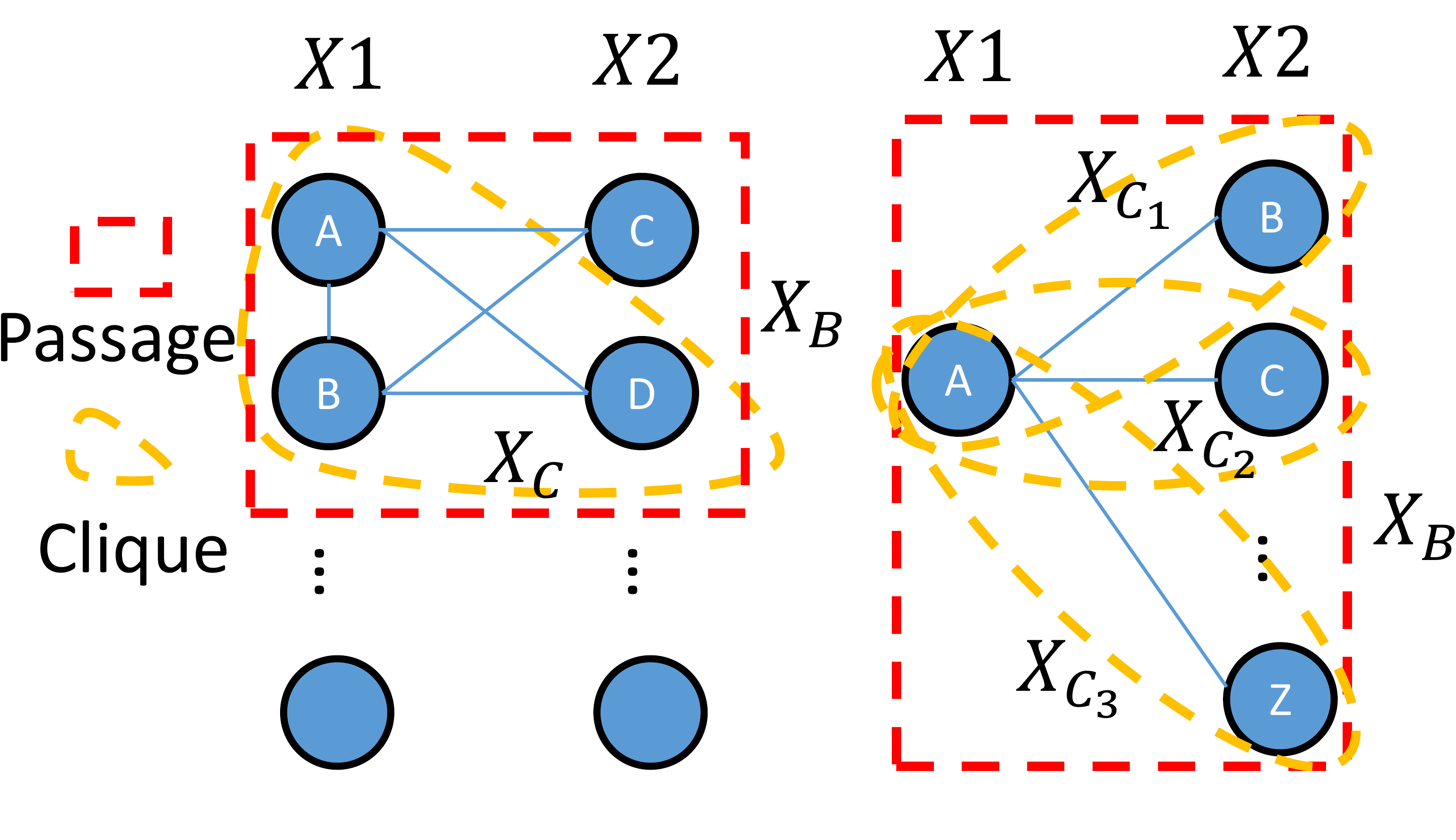}
	\caption{(Left) ABCD and (Right) AB\dots Z are two passages. }
	\label{fig.passage}
\end{figure}
Now, we formally define a graph $G = \langle X, E \rangle$, where $X$ is a set of random variables and $X = (X1, X2)$, i.e. $X1 \cap X2 = \emptyset, X1 \cup X2 = X$ and $X1, X2 \neq \emptyset$. The concept of PMN can now be defined.
\begin{defini}[PMN]
	\label{def.PMN}
	For a joint probability $P(X)$, $X = (X1, X2)$, if
	\begin{align}
	\label{eq.markov.1}
	P(X_i | X1 \cup X_{N(i)} \backslash X_i) = P(X_{i} | \backslash X_{i}), \forall X_i \in X1,\\
	\label{eq.markov.2}
	P(X_i | X2 \cup X_{N(i)} \backslash X_i) = P(X_{i} | \backslash X_{i}), \forall X_i \in X2,
	\end{align}
	then $P$ is a PMN with respect to $G$.
\end{defini}
The following proposition is a consequence of Definition \ref{def.PMN}, and an example is visualized in Figure \ref{fig.MNvsPMN}.
\begin{prop}
	If $P$ is an MN with respect to $G$, then $P$ is a PMN with respect to $G$, but not vice versa.
\end{prop}

\begin{prop}
	\label{prop.pairwise.MV}
	If $P$ is a PMN with respect to $G$, $X_u \in X1, X_v \in X2$, and $v \not \in N(u)$,  then $X_u \independent X_v | \backslash \{X_u, X_v\}$.
\end{prop}
See Appendix \ref{proof.prop.pairwise} for the proof.

The concept of \emph{Passage} is defined as follows:
\begin{defini}[Passage]
	Let $X = (X1, X2)$. We define a \textbf{passage} $B$ of $G$ as a subgraph of $G$, such that $X_B \cap X1 \neq \emptyset, X_B \cap X2 \neq \emptyset, $ and $\forall X_u \in (X1\cap X_B), \forall X_v \in (X2\cap X_B), $ we have edge $(X_u, X_v) \in E_B.$
\end{defini}
Here we highlight two of the passage structures of two graphs in Figure \ref{fig.passage}. 

From definition, we can see all cliques that go across two groups are passages, but not all passages are cliques:
\begin{prop}
	Let $X = (X1, X2)$. Given a passage $B$ of $G$, $B$ is a complete subgraph if and only if $\forall X_u, X_v \in X_B \cap X1,$ edge $(X_u, X_v) \in E_B$ and $\forall X_u, X_v \in X_B \cap X2, $ edge $(X_u, X_v) \in E_B.$
\end{prop}

As an analogy to a Gibbs distribution used in the Hammersley-Clifford Theorem, we define the Gibbs partitioned ratio.
\begin{defini}[Gibbs Partitioned Ratio]
	\label{def.GibbsMI}
	For a joint distribution $P$ over $X = (X1, X2)$, if the partitioned ratio has the form
	$$
	\frac{P(X1, X2)}{P(X1) P(X2)} = \frac{1}{Z}\prod_{B \in \mathbf{B}(G)} \phi_B(X_B),
	$$
	where $\mathbf{B}(G)$ is the set of all passages in $G$, then $\frac{P(X_{1}, X_{2})}{P(X1) P(X_{2})}$ is called the Gibbs partitioned ratio (GPR) over $G$. 
\end{defini}

\section{Factorization over Passages}
\label{sec.factorization}
In this section, we will investigate the question: can we have a similar factorization theorem like Theorems \ref{thm.HC} and \ref{thm.gibbs} for PMN? If so, learning the sparse factorization of PR may reveal the Markovian properties among random variables.
\subsection{Fundamental Properties}
There are two steps for introducing our factorization theorems.
The first step is establishing the Markovian property of random variables using the factorization of PR.
%\begin{figure}
%	\includegraphics[width=.49\textwidth]{MLMI_Illus}
%	\includegraphics[width=.49\textwidth]{MI_Illus}
%\end{figure}
\begin{them}
	\label{thm.gibbs.to.struct} 
		Given $X=(X1, X2)$, if PR $\frac{P(X1, X2)}{P(X1) P(X2)}$ is a GPR over a graph $G$ then 
		$P$ is a PMN with respect to $G$.
\end{them}
See Appendix \ref{sec.proof.gibbs.to.struct} for the proof.

Next, let us prove the other direction: From the Markovian property to the factorization.
\begin{them}
	\label{thm.struct.to.gibbs}
	Given  $X = (X1, X2)$, if $P$ is a PMN with respective to a graph $G$,
	then $\frac{P(X1,X2)}{P(X1)P(X2)}$ is a GPR over $G$.
\end{them}
See Appendix \ref{sec.proof.struct.to.gibbs} for the proof.

Simply, the factorization of a GPR is only related to the ``linkage'' (or rigorously, passages) between two groups. Interestingly, if we have an MN whose groups are linked via a few ``bottleneck'' passages, then the factorization is simply over those sparse passages, no matter how densely the graph are connected within each group. This gives PMN a significant advantage over traditional MN in terms of modelling: If the interactions between groups are simple (e.g. linear), we do not need to care the interactions within groups, even if they are highly complicated (e.g. non-linear).
For example, in the bipartisan analysis problem, a PR over congresspersons can be represented only via a few cross-party links, and a large chunk of connections between congresspersons within their own party can be ignored, no matter how complicated they are. 

Theorems \ref{thm.gibbs.to.struct} and \ref{thm.struct.to.gibbs} point out a promising direction for structural learning of a PMN: Once the sparse factorization of a GPR is learned, we are able to recover the \emph{sparse passages} of a PMN partitioned into two groups. 
%However, the construction of the potential function in Theorem \ref{thm.struct.to.gibbs} implies a challenge for us to model such a GPR in practice: If a passage involves many nodes (see Figure \ref{fig.passage}, right), its parameterization will take up an enormous amount of memory resources. 
%
%The above issue motivates us to explore some simplifications of general GPRs.

\subsection{Simplification of Passage Factorization}
The Hammersley-Clifford theorem (Theorem \ref{thm.HC})
 shows $P$ factorizes over cliques of $G$, given $P$ is an MN with respect to $G$. However, if one does not know the maximum size of cliques, the model of a probability function has to consider factors on all potential cliques, i.e.,  all subsets of $X$. It is unrealistic to construct a model with $2^{|X|}$ factors under the high-dimensional setting. 
 
 Therefore, a popular assumption called ``pairwise MN'' \citep{PGM_Koller,Murphy:2012:MLP:2380985} has been widely used to lower the computational burden of MN structure learning. It assumes that in $P$, all clique factors can be further recovered using only bivariate and univariate components which give rise to a pairwise model with only $(|X|^2+|X|)/2$ factors.
 Some well known MNs, such as Gaussian MN and Ising model are all examples of pairwise MNs.

Similar issues also happen when modelling GPR. There are $(2^{|X1|} - 1)(2^{|X2|} - 1)$ possible passage potentials for the set of random variables $X = (X1, X2)$.
Following the same spirit, we can consider a simplified model of PR by assuming that all passage potentials of the GPR must factorize in a pairwise fashion, i.e.:
\begin{defini}[Pairwise PR]
	\label{def.PGPR}
	For a joint distribution $P$ over $X = (X1, X2)$, if the partitioned ratio has the form
	\begin{align*}
		\frac{P(X1, X2)}{P(X1) P(X2)} &= \frac{1}{Z}\prod_{B \in \mathbf{B}(G)} \phi_B(X_B)\\
		&= \frac{1}{Z} \prod_{B \in \mathbf{B}(G)} \prod_{X_u, X_v \in X_B, u\le v} h_{u,v}(X_u, X_v),
	\end{align*}
	then $\frac{P(X_{1}, X_{2})}{P(X1) P(X_{2})}$ is called the pairwise Gibbs partitioned ratio (pairwise PR) over $G$. 
\end{defini}
If we can assume the GPR we hope to learn is also a pairwise PR, the model may only contain $(|X|^2 + |X|)/2$ pairwise factors, and is much easier to construct. 

In fact, pairwise PR does not have straightforward relationship with pairwise MN, i.e., a PR of a pairwise MN may not be a pairwise PR, meanwhile the joint distribution corresponding to a pairwise PR may not be a pairwise MN, since the pairwise MN and the pairwise PR apply the same assumption on the parameterizations of two fundamentally different quantities, the joint probability and the PR respectively. 

Whether one should impose such an assumption on joint probability or PR is totally up to the application, as neither parameterization is  always superior to the other. If the application focuses on learning the connections between two groups, we believe imposing such an assumption on PR directly is more sensible.

%Here we give a \emph{sufficient condition} when GPR is a pairwise GPR.
%\begin{prop}
%	\label{prop.pairwise}
%	If $P(X), P(X1), P(X2)$ over $X, X1$ and $X2$ all factorizes on pairwise potentials, then a GPR $\frac{P(X1,X2)}{P(X1)P(X2)}$ is a pairwise GPR.
%\end{prop}
%This is easy to see from the construction of the potential function \eqref{eq.delta} in the proof of Theorem \ref{thm.struct.to.gibbs}.
However, as a special case, a joint Gaussian distribution is a pairwise MN, and its PR is also a pairwise PR.
\begin{prop}
	If $P$ over $X = (X1, X2)$ is a zero-mean Gaussian distirbution, then the PR $\frac{P(X1,X2)}{P(X1)P(X2)}$ is a pairwise PR. 
\end{prop}
Since the Gaussian distribution factorizes over pairwise potentials, and the marginal distribution $P(X1)$ and $P(X2)$ are still Gaussian distributions. From the construction of the potential function \eqref{eq.delta} in the proof of Theorem \ref{thm.struct.to.gibbs}, we can verify this statement. Moreover, one can show it has the pairwise factor $h_{u,v}(X_u, X_v ) = \exp(\theta_{u,v} \cdot X_uX_v)$, where $\theta_{u,v}$ is the parameter.

This pairwise assumption together with factorization theorems motivate us to recover the structure of PMN by learning a sparse pairwise PR model: For any $ X_u \in X1, X_v \in X2$, if $X_u, X_v$ appear in the same pairwise factor of a PR model, they must be at least involved in one of the passage potentials.

%
%However, in general
%
%Proposition \ref{prop.pairwise} is merely a sufficient condition, and there are many other 
%
%%\begin{align*}
%%	\frac{P(X)}{P(X1)P(X2)} =& \frac{1}{Z} \prod_{B \in \mathbf{B}(G)} \phi_B(X_B) 
%%	\\=& \frac{1}{Z} \prod_{B \in \mathbf{B}(G)} \prod_{X_i, X_j \in X_B} g(X_i, X_j).
%%\end{align*}
\section{Estimating PR from Samples}
\label{sec.algorithm}
%{\color{blue}{define some notations}}
%Before start, we list a few notations: A sample drawn from $P$ is denoted as $\boldx, \boldx \in \mathbb{R}^p$. Two subvectors of $\boldx$ $\boldx_1$ and $\boldx_2$ correspond to the random variables $X_1$ and $X_2$. 
%Density functions of marginal distributions $P(X_1)$ and $P(X_2)$ are denoted as $p(\boldx_1)$ and $p(\boldx_2)$. Sample approximated functionals $Q$ will be written as $Q^n$.

%Theorems \ref{thm.gibbs.to.struct} and \ref{thm.struct.to.gibbs} show that the Markovian property among random variables leads to \emph{direct} parameterizations of a PR and a  PR parameterization also indicates the Markovian property between two groups of the random variables. This inspires us to estimate a model of the PR \emph{directly} in order to discover the MN structure between groups.
%The facroization theorm Theorem \ref{thm.struct.to.gibbs} and \ref{thm.gibbs.to.struct} together with the simplification of pairwise Assumption

To estimate PR using such a model, we require a set of samples \[\{ \boldx^{(i)}\}_{i=1}^n \iid P, ~~~ \boldx \in \mathbb{R}^m,\] and each sample vector $\boldx^{(i)}$ is a joint sample, i.e. $\boldx^{(i)} = \left(\boldx_1^{(i)}, \boldx_2^{(i)}\right)$ where $\boldx_1, \boldx_2$ are subvectors corresponding to two groups. 

We define a log-linear pairwise PR model $g(\boldx;\boldtheta)$:
\begin{align*}
g(\boldx;\boldtheta) := \frac{1}{N(\boldtheta)}\exp[\sum_{ u\le v} \boldtheta_{u,v }^\top \boldpsi(\boldx_{u,v})],
\end{align*}
where $\boldtheta_{u,v} \in \mathbb{R}^b$ is a column vector,
\[\boldtheta= (\boldtheta^{\top}_{1,2},\ldots, \boldtheta^{\top}_{1,m},\boldtheta^{\top}_{2,3},\ldots,\boldtheta^{\top}_{2,m},\ldots,\boldtheta^{\top}_{m-1,m})^\top,\] and $\boldpsi$ is a vector valued feature function $\boldpsi: \mathbb{R}^2 \rightarrow \mathbb{R}^b$. 
Notice that we still have to model all pairwise features in $\boldx$, but the vast majority of these pairs are going to be nullified due to Theorem \ref{thm.struct.to.gibbs} if links between two groups are sparse. 

$N(\boldtheta)$ is defined as a normalization function of $g(\boldx;\boldtheta)$:
\begin{align}
\label{eq.norm.popu}
N(\boldtheta) := \int{ p(\boldx_1)p(\boldx_2) \exp[\sum_{ u\le v} \boldtheta_{u,v }^\top \boldpsi(\boldx_{u,v})]} \dx,
\end{align}
where $p(\boldx_1)$ and $p(\boldx_2)$ are the marginal distributions of $p(\boldx)$, so it is guaranteed that 
$
\int p(\boldx_1)p(\boldx_2) g(\boldx;\boldtheta) \dx = 1.
$

$N(\boldtheta)$ in \eqref{eq.norm.popu} can be approximated via two-sample U-statistics \citep{hoeffdingInequality} using the dataset, 
\begin{align*}
N(\boldtheta) \approx \hat{N}(\boldtheta):= \frac{1}{{n \choose 2}}\sum_{j\neq k} {\exp[\sum_{ u\le v} \boldtheta_{u,v }^\top \boldpsi(\boldx^{[j,k]}_{u,v})]},
\end{align*}
where $\boldx^{[j,k]}$ is a \emph{permuted sample}: $\boldx^{[j, k]} = (\boldx^{(j)}_1, \boldx^{(k)}_2)$.
%Notice that this term can always be easily normalized for \emph{any choice} of $\boldpsi$.

Notice that the normalization term $N(\boldtheta)$ in \eqref{eq.norm.popu} is an integral with respect to a probability distribution $p(\boldx_1)p(\boldx_2)$. Though we do not have samples directly from such a distribution, U-statistics help us ``simulate'' such an expectation using joint samples. 
In Maximum Likelihood Estimation, 
density models are in general hard to compute since their normalization term is not with respect to a sample distribution. In comparison, $N(\boldtheta)$ can always be easily approximated for \emph{any choice} of $\boldpsi$. This gives us the flexibility to consider complicated PR models beyond the conventional Gaussian or Ising models. 

This model can be learned via the algorithm of maximum likelihood mutual information (MLMI) \citep{SuzukiMLMI}, by simply minimizing the Kullback-leibler divergence between $p(\boldx)$ and $p_\boldtheta(\boldx)=p(\boldx_1)p(\boldx_2)g(\boldx;\boldtheta)$:
\begin{align*}
\hat{\boldtheta} = \argmin_{\boldtheta} \text{KL}[p||{p_\boldtheta}].
\end{align*}
Substitute the model of $g(\boldx;\boldtheta)$ into the above objective and approximate $N(\boldtheta)$ by $\hat{N}(\boldtheta)$, then the estimated parameter $\hat{\boldtheta}$ is obtained as
\begin{align*}
\hat{\boldtheta}=\argmin_\boldtheta \underbrace{ -\sum_{i=1}^{n}\sum_{ u\le v} \boldtheta_{u,v }^\top \boldpsi(\boldx_{u,v}^{(i)}) +\log \hat{N}(\boldtheta)}_{ \ell_{\text{MLMI}}(\boldtheta)} + C,
\end{align*}
where $C$ is some constant. From now on, we denote $\ell_{\text{MLMI}} (\boldtheta)$ as the \emph{negative} likelihood function.
Due to Theorem \ref{thm.struct.to.gibbs} and our parametrization, if the passages between two groups are rare, then $\boldtheta$ is very sparse. Therefore, we may use sparsity inducing group-lasso penalties \citep{YuanLi2006GroupLasso} to encourage the sparsity on each subvector $\boldtheta_{u,v}$:
\begin{align}
\label{eq.the.objective}
\hat{\boldtheta}  = \argmin_\boldtheta  \ell_{\text{MLMI}} (\boldtheta) + \lambda\sum_{u\le v} \|\boldtheta_{u,v}\|.
\end{align}
This objective is convex, unconstrained, and can be easily solved by standard sub-gradient methods. $\lambda$ is a regularization parameter that can be tuned via cross-validation. 

Now let us define the ``true parameter'' $\boldtheta^*$, such that $  p(\boldx)=q(\boldx)g(\boldx;\boldtheta^*).$
The learned parameter $\hat{\boldtheta}$ is an estimate of $\boldtheta^*$, where $\boldtheta^*_{u,v}$ is non-zero  on pairwise features that are at least involved in one of the passage potentials. Moreover, as Theorem \ref{thm.gibbs.to.struct} and Proposition \ref{prop.pairwise.MV} show, if $X_u \in X1$ and $X_v \in X2$ are not in any of the passage structures, i.e., $\boldtheta^*_{u,v} = \boldzero$, then $X_u \independent X_v | \backslash \{X_u, X_v\}$. 

Given the optimization problem \eqref{eq.the.objective}, it is natural to consider the structure recovery consistency, i.e., under what conditions, the sparsity pattern of $\hat{\boldtheta}$ is the same as that of $\boldtheta^*$?  

\section{High-dimensional Structure Recovery Consistency}
\label{sec.consistency}
To better state the structure recovery consistency theorem, we use new indexing system with respect to the sparsity pattern of the parameter. 
Denoting the pairwise index set as $H = \{(u,v) | u\ge v\}$, 
 two sets of \textit{subvector indices} can be defined as  $S = \{t'\in H ~|~ \|\boldtheta^*_{t'}\| \neq 0\}, S^c = \{t'' \in H ~|~ \|\boldtheta^*_{t''}\| = 0\}.$ We rewrite the objective \eqref{eq.the.objective} as
\begin{align}
\label{eq.obj.alter}
\thetahat = \argmin_{\boldtheta} \ell(\boldtheta) 
+ \lambda_{n} \sum_{t'\in S} \| \boldtheta_{t'} \| + \lambda_{n} \sum_{t''\in S^c}\|\boldtheta_{t''}\|.
\end{align}
Similarly we can define 
$\hat{S}$ and $\hat{S^c}$. From now on, we simplify $\ell_{\mathrm{MLMI}}(\boldtheta^*)$ as $\ell(\boldtheta^*)$.

Now we state our assumptions.
\begin{assum}[Dependency]
	\label{assum.depen}
	The minimum eigenvalue of the \textbf{submatrix} of the log-likelihood Hessian is lower-bounded:
	\begin{align*}
	\Lambda_\mathrm{min}(\nabla_{\boldtheta_S}\nabla_{\boldtheta_S} \ell(\boldtheta^*)) \ge \lambda_\mathrm{min} > 0,
	\end{align*}
	with probability 1, where $\Lambda_{\mathrm{min}}$ is the minimum-eigenvalue operator of a symmetric matrix
\end{assum}
\begin{assum}[Incoherence]
	\label{assum.incoherence}
	\begin{align*}
	\max_{t'' \in S^c}\left\|\left[\nabla_{\boldtheta_{t''}}\nabla_{\boldtheta_S}\ell(\boldtheta^*)\right] \left[\nabla_{\boldtheta_S}\nabla_{\boldtheta_{S}} \ell(\boldtheta^*)\right]^{-1}\right\|_1 \le 1-\alpha,
	\end{align*} 
	with probability 1, where $0<\alpha \le 1, $ and $\|Y\|_1 = \sum_{i,j} \|Y_{i,j}\|_1$.
\end{assum}
The first two assumptions are common in the literatures of support consistency. The first assumption guarantees the identifiability of the problem. The second assumption ensures the pairwise factors in passages are not too easily affected by those are not in any passages. The third assumption states the likelihood function is ``well-behaved''.
\begin{assum}[Smoothness on Likelihood Objective]
	\label{assum.smooth}
	The log-likelihood ratio $\ell(\boldtheta)$ is smooth around its optimal value, i.e., it has bounded derivatives
	\begin{align*}
	\max_{\bolddelta, \|\bolddelta\|\leq \|\boldtheta^*\|}\left\| \nabla^2 \ell(\boldtheta^*+\bolddelta)\right\|  &\leq \lambda_\mathrm{max} < \infty,\notag\\
	\max_{t\in \{S \cup S^c\}} \max_{\bolddelta, \|\bolddelta\|\leq \|\boldtheta^*\|}  \vertiii{\nabla_{\boldtheta_t}\nabla^2 \ell(\boldtheta^* + \bolddelta)}&\leq \lambda_{3,\mathrm{max}}< \infty,\notag
	\end{align*}
	with probability $1$. 
\end{assum}
$\left\|\cdot\right\|$,  $\vertiii{\cdot}$ are the spectral norms of a matrix and a tensor respectively (See e.g., \citet{tomioka2014spectral} for the definition of the spectral norm of a tensor).

\begin{assum}[Bounded PR Model]
	\label{assum.bounded.ratio}
	For any vector $\bolddelta \in \mathbb{R}^{\text{dim}(\boldtheta^*)}$ such that $\|\bolddelta\|\leq \|\boldtheta^*\|$, the following inequality holds:
	\begin{align*}
	0 < C_\mathrm{min} \le g(\boldx; \boldtheta) \le C_\mathrm{max} < \infty,
	\end{align*}
	 $\|\boldf_t\|_{\infty} \le \frac{C_{\boldf_t, \mathrm{max}}}{\sqrt{b}}$ and $\|\boldf_t\| \le C'_{\boldf_t, \mathrm{max}}, \forall t \in (S \cup S^c)$.
\end{assum}
This assumption simply indicates our PR model is bounded from above and below around the optimal value. Though it rules out the Gausssian distribution whose PR is not necessarily upper/lower-bounded, as a theory of generic pairwise models, we think it is acceptable.

\begin{them}
	\label{corol.bounded.ratio}
	Suppose that Assumptions  \ref{assum.depen}, \ref{assum.incoherence}, \ref{assum.smooth}, and  \ref{assum.bounded.ratio} are satisfied as well as $\min_{t\in S} \|\boldtheta^*_t\| \geq \frac{10}{\lambda_\text{min}} \sqrt{|S|}\lambda_{n}$.
	Suppose also that the regularization parameter is chosen so that
	\begin{align*}
	\frac{24(2-\alpha)}{\alpha} \sqrt\frac{{M\log \frac{m^2+m}{2}}}{n} &\le \lambda_{n}, 
	\end{align*}
	where $M$ is a positive constant.
	Then there exist some constants $L$, $K_1$ and $K_2$ such that if $n\geq L |S|^2\log  \frac{m^2+m}{2}$, with the probability at least $1- K_1\exp\left( - K_2 \lambda_{n}^2 n \right)$, MLMI in \eqref{eq.obj.alter} has the following properties:
	\vspace*{-3mm}
	\begin{itemize}
		\item Unique Solution: The solution of \eqref{eq.obj.alter} is unique.
		\item Successful Passage Recovery: $\hat{S} = S$ and $\hat{S^c} = {S^c}$.
		\item $\|\hat{\boldtheta} - \boldtheta^*\| = \mathcal{O}(\sqrt{\frac{\log \frac{m^2 + m}{2}}{n}})$.
	\end{itemize}
\end{them}
\vspace*{-3mm}
The proof of Theorem $\ref{corol.bounded.ratio}$ is detailed in Appendix \ref{proof.main.them}. Since the PR function is a \emph{density ratio function} between $p(\boldx)$ and $p(\boldx_1)p(\boldx_2)$, and \eqref{eq.obj.alter} is also a sparsity inducing Kullback-Leibler Importance Estimation Procedure (KLIEP) \citep{Covariate_Shift_jour}, the previously developed support consistency theorem  \cite{Liu2015Change,arxiv_onMNChange} can be applied here as long as we can verify a few assumptions and lemmas. 

The sample size required for the proposed method increases with $\log m$ (since $\log \frac{(m^2+m)}{2} \le 2\log m$ if $m>2$ ) and the estimation error on $\boldtheta$ vanishes at the speed of $\sqrt{\frac{\log m}{n}}$. They are the same as the optimal rates obtained in previous researches for Gaussian graphical model structure learning \citep{Ravikumar_2010,RaskuL1GMRF}.

This theorem also indicates that the sample size required is not influenced by the structural density of the entire MN structure, but by the number of pairwise factors in the passage potentials. This is encouraging since we are allowed to explore  PMNs with dense groups which would be hard to learn using conventional methods.

\section{Experiments}
\label{sec.exp}
Unless specified otherwise, we use pairwise feature function $\psi(x_u, x_v) = x_ux_v$. Note this does \emph{not} mean we assume the Gaussianity over the joint distribution, since this is a parameterization of a PR rather than a joint distribution.
\subsection{Synthetic Datasets}
\label{sec.syn}
\begin{figure*}
	\centering
	\subfigure[$\boldTheta, \rho = .8, |X1| = 40$]{	\includegraphics[width=.27\textwidth]{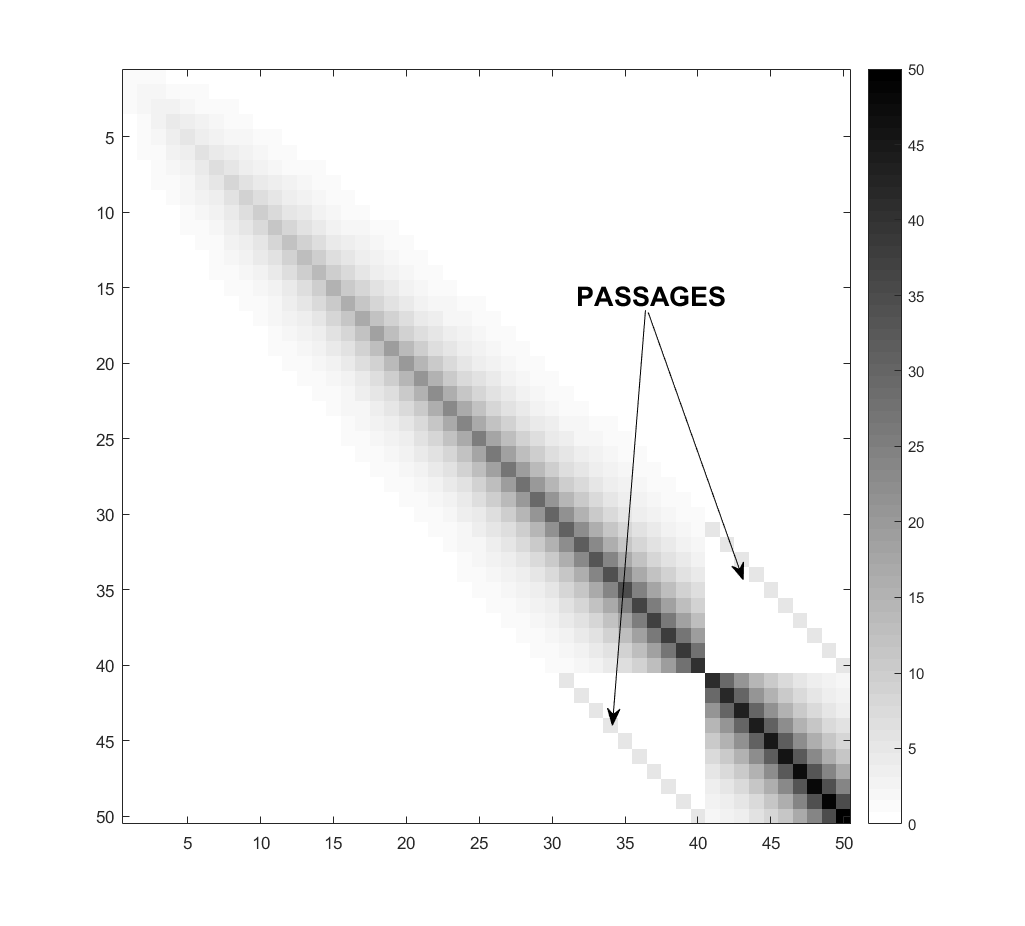}
		\label{fig.precision}
	}
	\subfigure[]{\includegraphics[width=.10\textwidth]{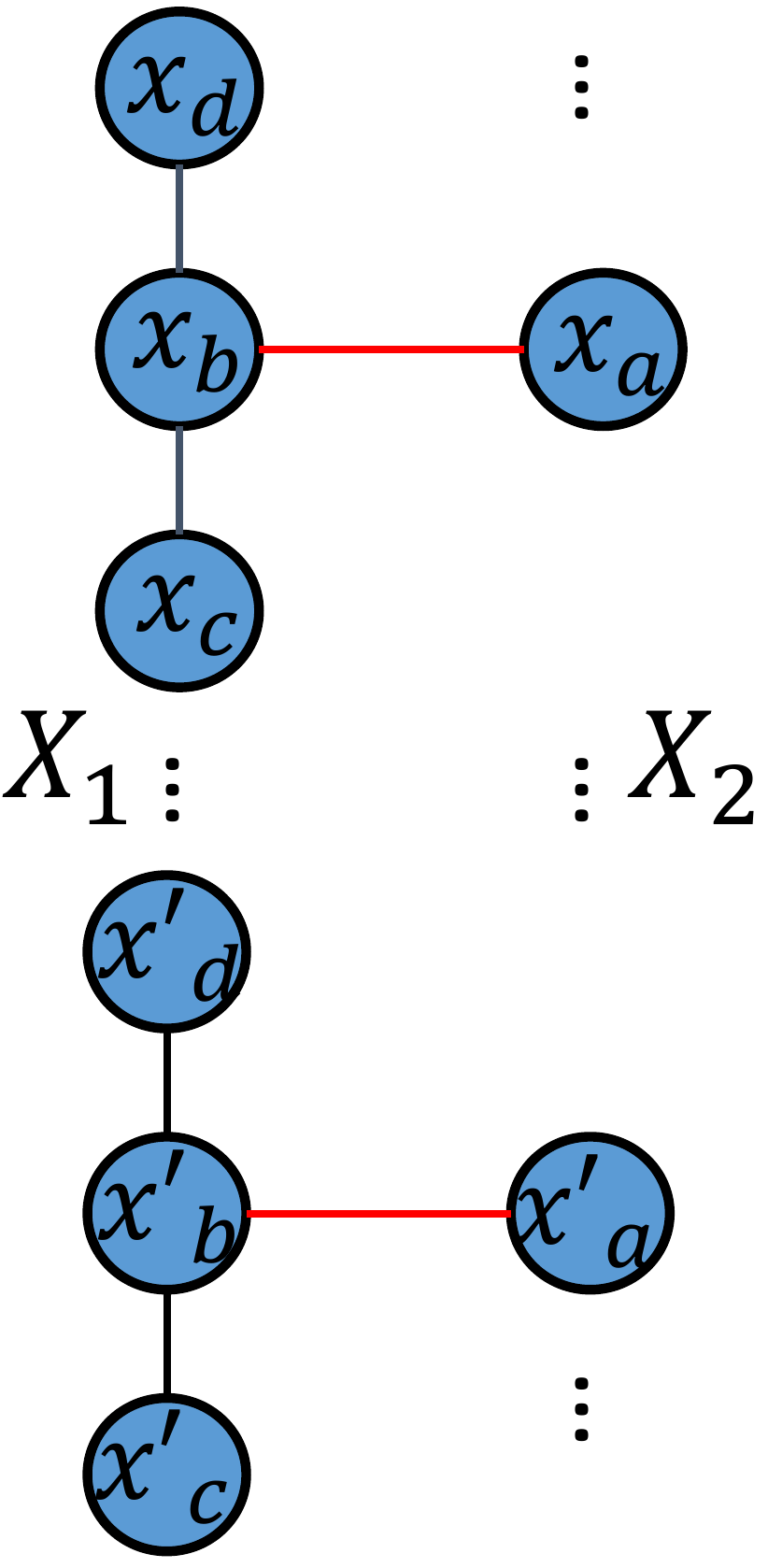}
		\label{fig.T}
	}
	\subfigure[ROC of Gaussian Dataset]{\includegraphics[width=.27\textwidth]{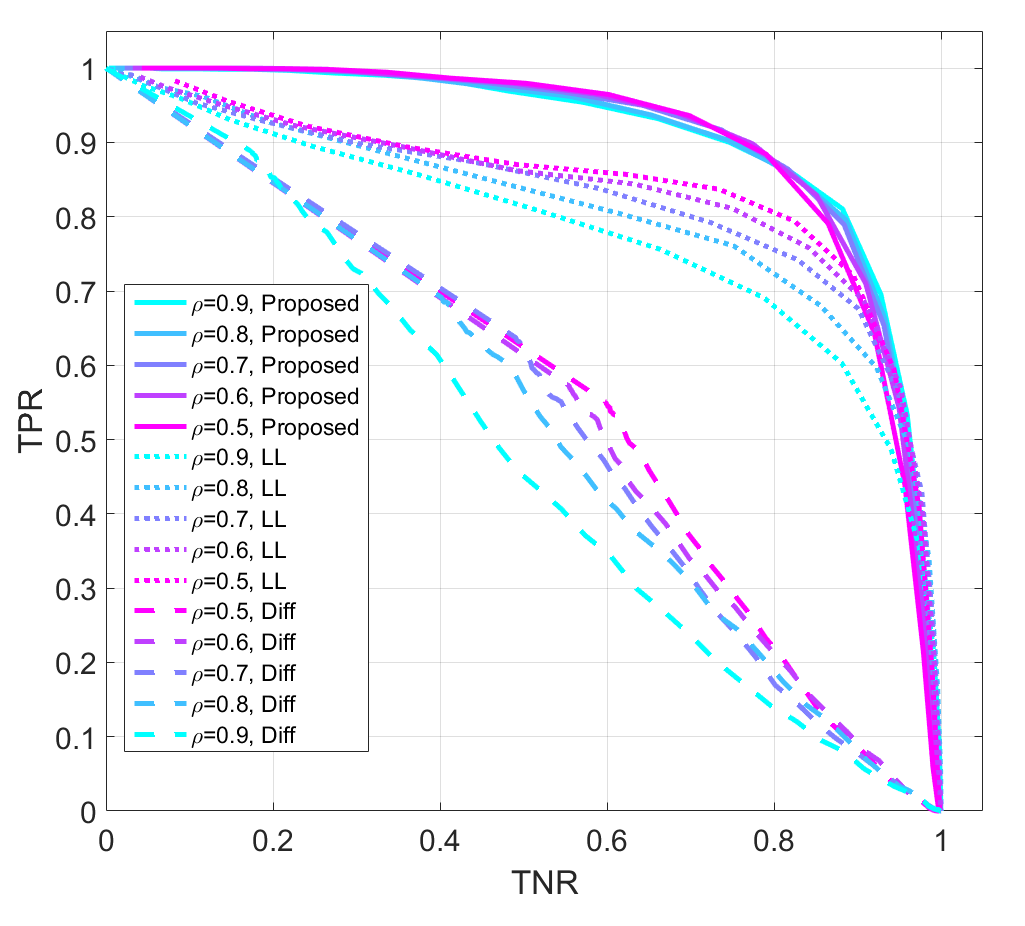}
	\label{fig.ROC}
	}
	\subfigure[ROC of ``Diamond'' Dataset]{\includegraphics[width=.27\textwidth]{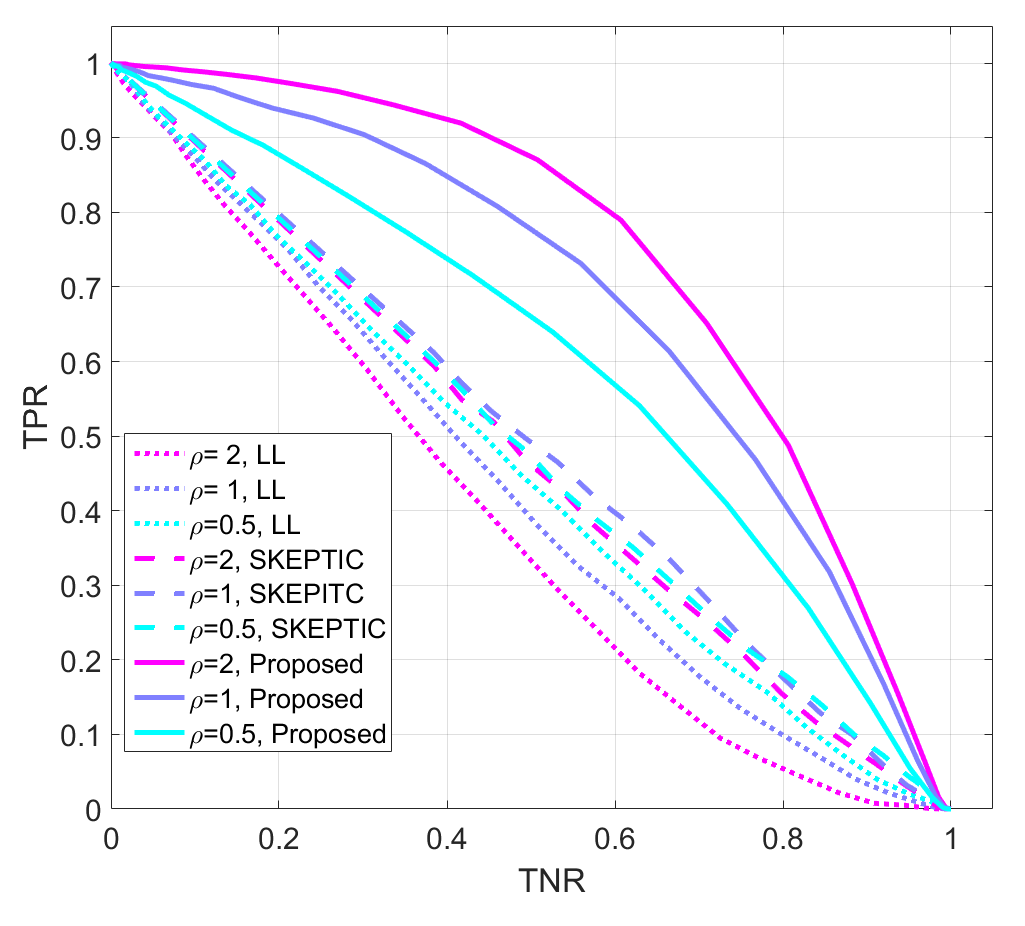}
		\label{fig.ROC2}
	}
	\caption{Synthetic experiments}
	\vspace*{-3mm}
\end{figure*}
We are interested in comparing the proposed method with a few possible alternatives:
\textbf{LL} \citep{MeinshausenBuhlmann,Ravikumar_2010}, \textbf{SKEPTIC} \citep{liu2012high} and 	\textbf{Diff} \citep{zhao2014direct}: 
A \emph{direct} difference estimation method that learns the differences between two MNs without learning each individual precision matrix separately. In this paper, we employed this method to learn the differences between two Gaussian densities: $p(\boldx)$ and $p(\boldx_1)p(\boldx_2)$. 
%	\textbf{LL} \citep{MeinshausenBuhlmann,Ravikumar_2010}: 
%	A Gaussian likelihood-based MN structure learning method that performs the node-wise regression among covariates.
%	\textbf{Diff} \citep{zhao2014direct}: 
%	A \emph{direct} difference estimation method that learns the differences between two MNs without learning each individual precision matrix separately. In this paper, we employed this method to learn the differences between two Gaussian densities: $p(\boldx)$ and $p(\boldx_1)p(\boldx_2)$. 
%	\textbf{SKEPTIC} \citep{liu2012high}: 
%	A semi-parametric approach uses nonparametric rank-based correlation coefficient estimators such as Spearman’s rho and Kendall’s tau to replace the conventional Gaussian models.

We first generate a set of joint samples $\{\boldx^{(i)}\}_{i=1}^{50} \sim \mathcal{N}(\boldzero, \boldTheta^{-1})$, where $\boldTheta \in \mathbb{R}^{50 \times 50}$ and is constructed in two steps. First, create
\begin{align*}
\Theta_{i,j} = 
\begin{cases}
\rho^{|i-j|}\sqrt{ij}, & i,j<40 ~\mathrm{ or }~ i,j>40, \\
0, & \mathrm{Otherwise},
\end{cases}
\end{align*}
where $0<\rho<1$ is a coefficient controlling the dominance of the diagonal entries. Second, let $\Lambda$ be the $15^{\mathrm{th}}$ smallest eigenvalue of $\boldTheta$, and fill the submatrices $\boldTheta_{\{41,\dots, 50\}, \{31,\dots, 40\}}$ and $\boldTheta_{\{31,\dots, 40\}, \{41,\dots, 50\}}$ with $ \Lambda\boldI_{10}$, where $\boldI_{10}$ is a $10 \times 10$ identity matrix. By such a construction, we have created two groups over $X$: $X = (X_{\{1,\dots,40\}}, X_{\{41,\dots,50\}})$ and 10 passages between them. Notably, within two groups, the precision matrix is \emph{dense}, and random variables interact with each other via powerful links when $\rho$ is large. An example of $\boldTheta$ when $\rho=0.8$ is plotted in Figure \ref{fig.precision}. 
%For all methods, we generate a set of joint samples $\{\boldx^{(i)}\}_{i=1}^{50} \sim \mathcal{N}(\boldzero, \boldTheta^{-1})$, where $\boldTheta \in \mathbb{R}^{50 \times 50}$ and the power of interactions within each group is controlled by a coefficient $\rho$. A visualization of $\boldTheta$ is given in Figure \ref{fig.precision}.
We measure the performance of three methods using the True Postive Rate (TPR) and True Negative Rate (TNR). The detailed definition of TPR and TNR is deferred to Appendix, \ref{exp.setting}.

The ROC curve in Figure \ref{fig.ROC} can be plotted by adjusting the sensitivity of each method: Tuning the regularization parameter of the proposed method and LL, or the threshold parameter of Diff.

As we can see, the proposed method has the best overall performance on all $\rho$ choices, comparing to both LL and Diff. Also, as the links within each group get more and more powerful (by increasing $\rho$), the performance of LL and Diff decay significantly, while the proposed method almost remain unchanged.

As the proposed method is capable of handling complex models, we draw 50 samples from a 52-dimensional ``diamond'' distribution used in \citep{liu2014ChangeDetection} where the correlation among random variables are non-linear. 
To speed-up the sampling procedure, the graphical model of this distribution is constructed by concatenating $13$ simple $4$-variable MNs whose density functions are defined as 
\[
p(x_a, x_b, x_c, x_d) \propto \exp\left(-\rho x_a^2x_b^2 - .5x_b x_c - .5x_b x_d \right)\cdot \mathcal{N},
\]
where $\mathcal{N}$ is short for a normal density $\mathcal{N}(\boldzero, .5\boldI_4)$ over $x_a, x_b, x_c$ and $ x_d$. Notice this distribution does not have a closed form normalization term. 
%Since sampling from a high-dimensional irregular distribution can be time-consuming, we construct such a distribution by obtaining $nm/4$ samples from a 4-dimensional diamond distribution and regard the first $n$ samples as they were from dimension $1-4$, and the second $n$ samples were from dimension $5-8$ and so on. 
%In such a way, we can create a high-dimensional ``diamond'' distribution using many identical and independent sub-distributions. The construction of such dataset is detailed in the appendix.
The graphical model of such a distribution is illustrated in Figure \ref{fig.T}. In this experiment, the coefficient $\rho$ is used to control the strength of inter-group interactions ($x_a \leftrightarrow x_b$), and we set $\psi(x_u, x_v) = x_u^2 x_v^2$. Other than LL, we include SKEPTIC due to the non-Gaussian nature of this dataset.
The performance is compared in Figure \ref{fig.ROC2} using ROC curves. 

The correlation among random variables are completely non-linear. As the power of interactions on passages increases, LL performs worse and worse since it still relies on the Gaussian model assumption. Thanks to the correct PR model, the proposed method performs reasonably well and gets better when $\rho$ increases. As the density model does not fit into the Gaussian copula model, SKEPTIC also performs poorly. 

\subsection{Bipartisanship in $109^{\mathrm{th}}$ US Senate}
\begin{figure}
	\centering
	\includegraphics[width=.7\textwidth]{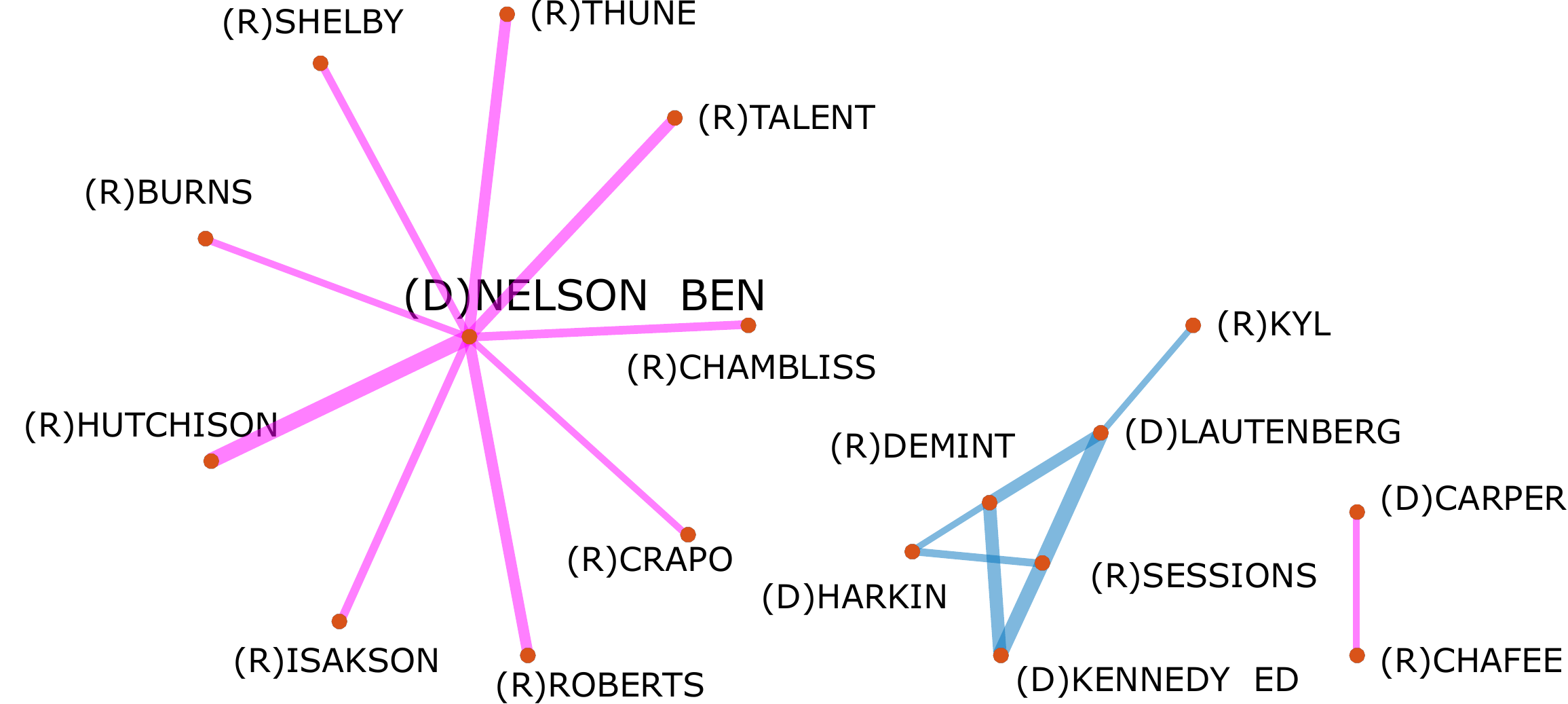}
	\caption{Bipartisanship in $109^{\mathrm{th}}$ US Senate. Prefix ``(D)'' or ``(R)'' indicates the party membership of a senator. Red: positive influence, Blue: negative influence. Edge widths are proportional to  $|\theta_{u,v}|$.}
	\label{fig.votes}
	\vspace*{-5mm}
\end{figure}
We use the proposed method to study the bipartisanship between Democrats and Republicans in the $109^{\mathrm{th}}$ US Senate via the recorded votes. There were totally 100 senators (45 Democrats and 55 Republicans) casting votes on 645 questions with ``yea'', ``nay'' or ``not voting''. The task is to discover the cross-party links between senators. We construct a dataset $\{\boldx^{(i)}\}_{i=1}^{645} \sim X$ using all 645 questions as observations, where each observation $\boldx \in \{1, -1, 0\}^{100}$ corresponds to the votes on a single question by 100 senators, and random variables $X = \left(X_{\{1,\dots, 45\}}, X_{\{46,\dots ,100\}}\right)$ are senators partitioned according to party memberships.

We run the proposed method directly on this dataset, and decrease $\lambda$ from 10 until $|\hat{S}| > 15$. To avoid complication, we only plot edges that contain nodes from different groups in Figure \ref{fig.votes}.

It can be seen that Ben Nelson, a conservative Democrat, who ``frequently voting against his party'' \citep{wiki:Ben_Nelson}, has multiple links with the other side. On the right, Democrat Tom Carper tends to agree with Republican Lincoln Chafee. Carper collaborated with Chafee on multiple bipartisan proposals \citep{carper1,carper2} while Chafee, who ``support for fiscal and social policies that often opposed those promoted by the Republican Party'' \cite{wiki:Lincoln-Chafee} finally switched his affiliation to Democratic in 2013. Interestingly, we have also observed a cluster of senators who tend to disagree with each other.
%and notably, none of the pairs of ``opponents'' are from the same party.

\subsection{Pairwise Sequences Alignment}
%\begin{figure*}[t]
%	\centering
%	\includegraphics[width = .7\textwidth]{sequence_matching}
%\end{figure*}
\begin{figure}
	\centering
	\subfigure[Twitter keyword frequency time-series alignments. $n = 50, m = 962 $ and $\mathcal{X} = \mathbb{R}$. ]{\includegraphics[width=.7\textwidth]{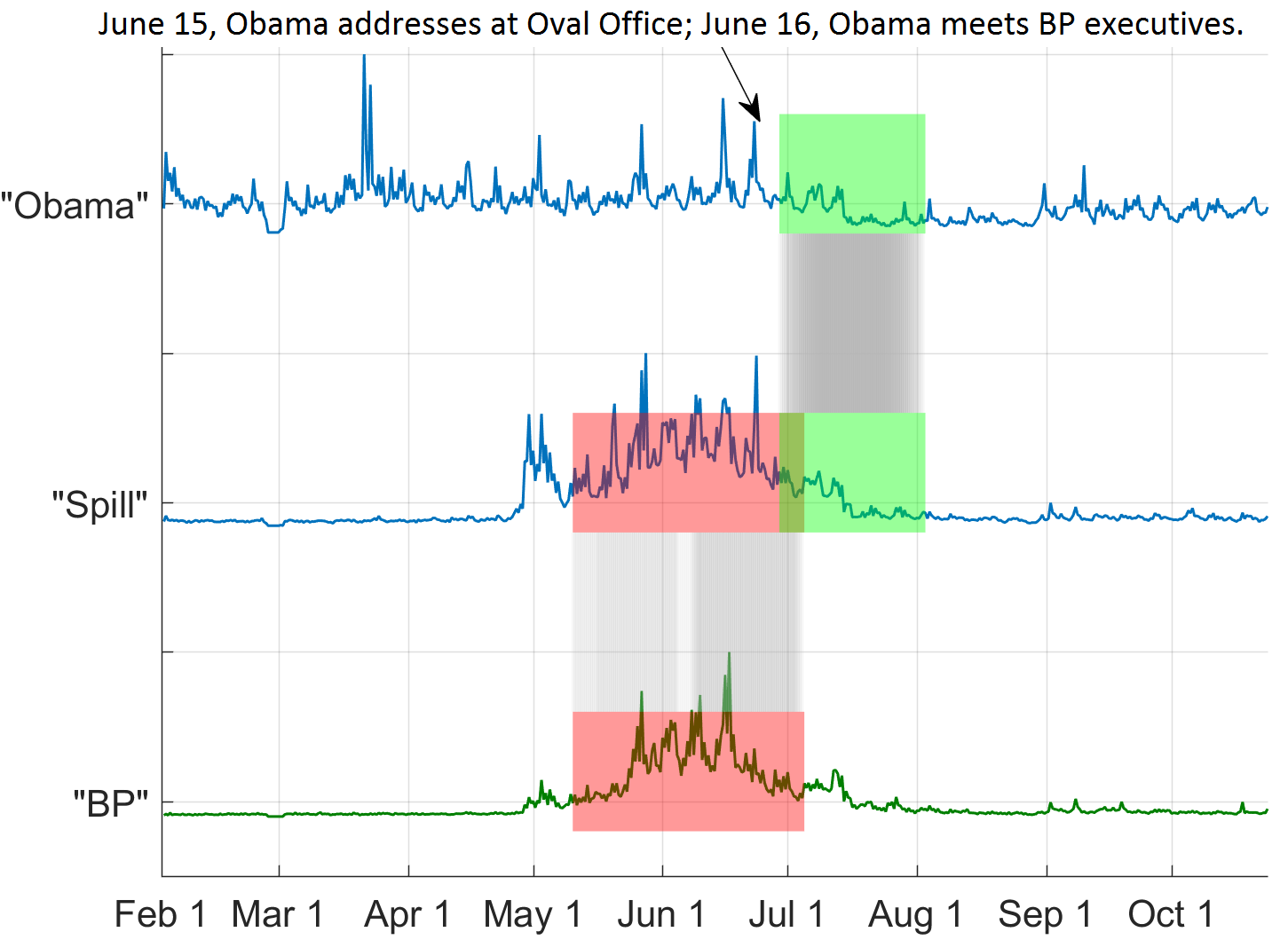}
		\label{fig.bp}}
	\subfigure[Amino acid sequence alignments between AAD01939 (human) and AAQ67266 (fly). $n = 10, m = 592,  \phi(x_i, x_j) = \delta(x_i, x_j)$ and $\mathcal{X} = \{\text{amino acid dictionary}\}$.]{\includegraphics[width=.7\textwidth]{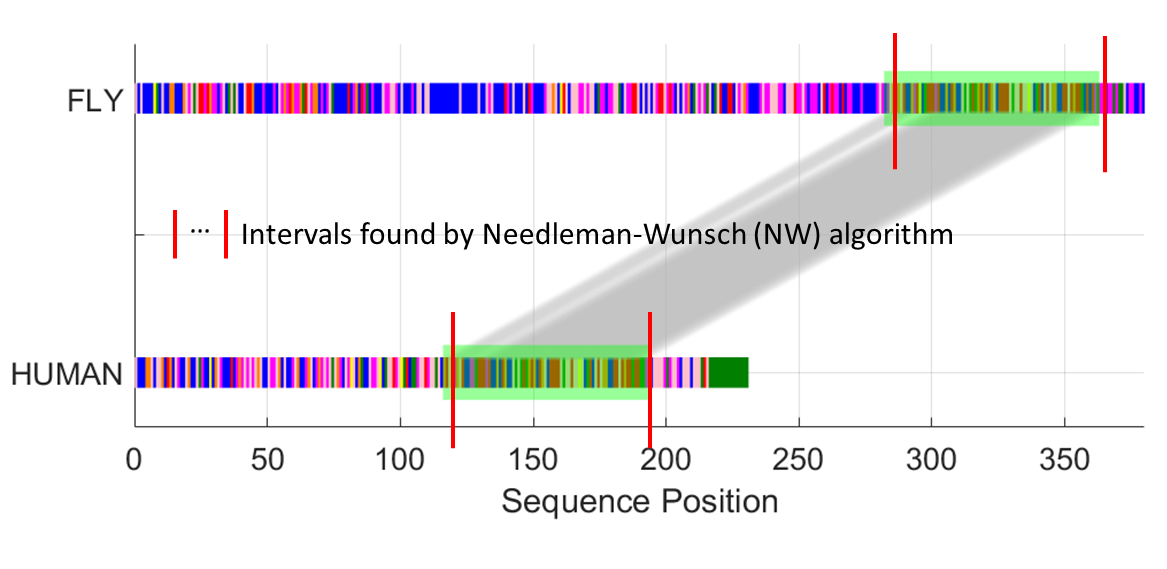}
	\label{fig.acid}}
	\caption{Sequence alignment. For two aligned windows with size $n$, we plot $n$ gray lines between two windows linking each pair of elements. Since lines are so close to each other, they look like ``gray shades'' on the plot. The color box contains the region of consecutively aligned windows.}
	\vspace*{-7mm}
\end{figure}
PMN can also be used to ``align'' sequences. Given \emph{a pair} of sequences where points are collected from the domain $\mathcal{X}$, we pick sequence 1 and construct the dataset by sliding a window sized $n$ toward future, until reaching the end. Suppose there are $m_1$ windows generated, then we can create a dataset $\{\boldx_1^{(i)}\}_{i=1}^n$, $\boldx \in \mathcal{X}^{m_1}$. Similarly, we construct another  dataset $\{\boldx_2^{(i)}\}_{i=1}^n$, $\boldx \in \mathcal{X}^{m_2}$ on sequence 2, and make joint samples by letting $\boldx^{(i)} = \left(\boldx_1^{(i)}, \boldx_2^{(i)}\right)$. 
After learning a PMN over two groups, 
if $X_u$ and $X_v$ are connected, then we regard the elements in the $u$-th window and the elements in the $v$-th window are ``aligned''.
See Figure \ref{fig.illus.seq.matching} in Appendix for an illustration. 

We run the proposed method to learn PMNs over two datasets: Twitter keyword count sequences \cite{liu2013changepoint} and Amino acid sequences with Genebank ID: AAD01939 and AAQ67266. The results were obtained by decreasing $\lambda$ from 10 so $|\hat{S}| > 15$. 

For the Twitter dataset, we collect normalized frequencies of keywords as time-series over 8 months, during the event "Deepwater Horizon oil spill"
%\footnote{\url{https://en.wikipedia.org/wiki/Deepwater_Horizon_oil_spill}}
 in 2010. We learn alignments between two pairs of keywords: ``Obama'' vs. ``Spill'' and ``Spill'' vs. ``BP''. The results are plotted in Figure \ref{fig.bp} where we can see the sequences of two pairs are aligned well  in chronological order. The two popular keywords, ``BP'' and ``Spill'' are synchronized throughout almost the entire event while ``Spill'' and ``Obama'' are only synchronized later on after he delivered his speech in Oval Office on this crisis on June 15th, 2010. 

The next experiment uses two amino acid string sequences, consisting codes such as `V', `I', `L' and `F', etc. Figure \ref{fig.acid} shows that the proposed method has successfully identified the aligned segment between eyeless gene of Drosophila melanogaster (a fruitfly) and human aniridia genes. The same segment is also spotted by widely used Needleman-Wunsch (NW) algorithm \citep{needleman1970general} with statistical significance. 

\section*{Acknowledgments}
SL acknowledges JSPS Grant-in-Aid for Scientific Research Activity Start-up 15H06823.
TS acknowledges JST-PRESTO.
MS acknowledges the JST CREST program. 
KF acknowledges JSPS Grant-in-Aid for Scientific Research on
Innovative Areas 25120012. 
Authors would like to thank four anonymous reviewers for their helpful comments. 
\bibliography{main}

\begin{thebibliography}{33}
\providecommand{\natexlab}[1]{#1}
\providecommand{\url}[1]{\texttt{#1}}
\expandafter\ifx\csname urlstyle\endcsname\relax
  \providecommand{\doi}[1]{doi: #1}\else
  \providecommand{\doi}{doi: \begingroup \urlstyle{rm}\Url}\fi

\bibitem[Banerjee et~al.(2008)Banerjee, {El Ghaoui}, and
  {d'Aspremont}]{Banerjee_Model_Selection}
Banerjee, O., {El Ghaoui}, L., and {d'Aspremont}, A.
\newblock Model selection through sparse maximum likelihood estimation for
  multivariate {G}aussian or binary data.
\newblock \emph{Journal of Machine Learning Research}, 9:\penalty0 485--516,
  March 2008.

\bibitem[Friedman et~al.(2008)Friedman, Hastie, and
  Tibshirani]{Friedman_GLasso}
Friedman, J., Hastie, T., and Tibshirani, R.
\newblock Sparse inverse covariance estimation with the graphical lasso.
\newblock \emph{Biostatistics}, 9\penalty0 (3):\penalty0 432--441, 2008.

\bibitem[Hammersley \& Clifford(1971)Hammersley and Clifford]{HC_them}
Hammersley, J.~M. and Clifford, P.
\newblock {M}arkov fields on finite graphs and lattices.
\newblock 1971.

\bibitem[Hoeffding(1963)]{hoeffdingInequality}
Hoeffding, W.
\newblock Probability inequalities for sums of bounded random variables.
\newblock \emph{Journal of the American statistical association}, 58\penalty0
  (301):\penalty0 13--30, 1963.

\bibitem[Koller \& Friedman(2009)Koller and Friedman]{PGM_Koller}
Koller, D. and Friedman, N.
\newblock \emph{Probabilistic Graphical Models: {P}rinciples and Techniques}.
\newblock MIT Press, Cambridge, MA, USA, 2009.

\bibitem[Liu et~al.(2009)Liu, Lafferty, and Wasserman]{nonparanormal}
Liu, H., Lafferty, J., and Wasserman, L.
\newblock The nonparanormal: Semiparametric estimation of high dimensional
  undirected graphs.
\newblock \emph{Journal of Machine Learning Research}, 10:\penalty0 2295--2328,
  2009.

\bibitem[Liu et~al.(2012)Liu, Han, Yuan, Lafferty, and Wasserman]{liu2012high}
Liu, H., Han, F., Yuan, M., Lafferty, J., and Wasserman, L.
\newblock High-dimensional semiparametric {G}aussian copula graphical models.
\newblock \emph{The Annals of Statistics}, pp.\  2293--2326, 2012.

\bibitem[Liu et~al.(2013)Liu, Yamada, Collier, and
  Sugiyama]{liu2013changepoint}
Liu, S., Yamada, M., Collier, N., and Sugiyama, M.
\newblock Change-point detection in time-series data by relative density-ratio
  estimation.
\newblock \emph{Neural Networks}, 43:\penalty0 72--83, 2013.

\bibitem[Liu et~al.(2014)Liu, Quinn, Gutmann, Suzuki, and
  Sugiyama]{liu2014ChangeDetection}
Liu, S., Quinn, J.~A., Gutmann, M.~U., Suzuki, T., and Sugiyama, M.
\newblock Direct learning of sparse changes in {M}arkov networks by density
  ratio estimation.
\newblock \emph{Neural Computation}, 26\penalty0 (6):\penalty0 1169--1197,
  2014.

\bibitem[Liu et~al.(2015)Liu, Suzuki, and Sugiyama]{Liu2015Change}
Liu, S., Suzuki, T., and Sugiyama, M.
\newblock Support consistency of direct sparse-change learning in {M}arkov
  networks.
\newblock In \emph{Proceedings of the Twenty-Ninth AAAI Conference on
  Artificial Intelligence (AAAI2015)}, 2015.

\bibitem[Liu et~al.(2016)Liu, Relator, Sese, Suzuki, and
  Sugiyama]{arxiv_onMNChange}
Liu, S., Relator, R., Sese, J., Suzuki, T., and Sugiyama, M.
\newblock {Support consistency of direct sparse-change learning in {M}arkov
  networks}.
\newblock \emph{Annals of Statistics}, 2016.
\newblock to appear.

\bibitem[Loh \& Wainwright(2012)Loh and Wainwright]{Poling-loh2012}
Loh, P.-L. and Wainwright, M.~J.
\newblock Structure estimation for discrete graphical models: Generalized
  covariance matrices and their inverses.
\newblock In Pereira, F., Burges, C.J.C., Bottou, L., and Weinberger, K.Q.
  (eds.), \emph{Advances in Neural Information Processing Systems 25}, pp.\
  2087--2095. 2012.

\bibitem[Meinshausen \& B{\"{u}}hlmann(2006)Meinshausen and
  B{\"{u}}hlmann]{MeinshausenBuhlmann}
Meinshausen, N. and B{\"{u}}hlmann, P.
\newblock High-dimensional graphs and variable selection with the lasso.
\newblock \emph{The Annals of Statistics}, 34\penalty0 (3):\penalty0
  1436--1462, 06 2006.

\bibitem[Murphy(2012)]{Murphy:2012:MLP:2380985}
Murphy, K.~P.
\newblock \emph{Machine Learning: A Probabilistic Perspective}.
\newblock The MIT Press, 2012.

\bibitem[Needleman \& Wunsch(1970)Needleman and Wunsch]{needleman1970general}
Needleman, S.~B. and Wunsch, C.~D.
\newblock A general method applicable to the search for similarities in the
  amino acid sequence of two proteins.
\newblock \emph{Journal of molecular biology}, 48\penalty0 (3):\penalty0
  443--453, 1970.

\bibitem[Press-Release({\natexlab{a}})]{carper1}
Press-Release.
\newblock Carper urges bipartisan compromise on clean air, {\natexlab{a}}.
\newblock URL
  \url{http://www.epw.senate.gov/pressitem.cfm?party=rep&id=230919}.

\bibitem[Press-Release({\natexlab{b}})]{carper2}
Press-Release.
\newblock Carper-chafee-feinstein will offer bipartisan budget plan similar to
  ``blue dog'' house proposal, {\natexlab{b}}.
\newblock URL
  \url{http://www.carper.senate.gov/public/index.cfm/pressreleases?ID=e5603ed2-80ae-483b-a98d-4aa9932edeaf}.

\bibitem[Raskutti et~al.(2009)Raskutti, Yu, Wainwright, and
  Ravikumar]{RaskuL1GMRF}
Raskutti, G., Yu, B., Wainwright, M.~J., and Ravikumar, P.
\newblock Model selection in {G}aussian graphical models: High-dimensional
  consistency of $\ell_1$-regularized mle.
\newblock In Koller, D., Schuurmans, D., Bengio, Y., and Bottou, L. (eds.),
  \emph{Advances in Neural Information Processing Systems 21}, pp.\
  1329--1336. Curran Associates, Inc., 2009.

\bibitem[Ravikumar et~al.(2010)Ravikumar, Wainwright, and
  Lafferty]{Ravikumar_2010}
Ravikumar, P., Wainwright, M.~J., and Lafferty, J.~D.
\newblock High-dimensional {I}sing model selection using $\ell_1$-regularized
  logistic regression.
\newblock \emph{The Annals of Statistics}, 38\penalty0 (3):\penalty0
  1287--1319, 2010.

\bibitem[Robert \& Casella(2005)Robert and Casella]{RobertMCStat2005}
Robert, C.~P. and Casella, G.
\newblock \emph{Monte Carlo Statistical Methods}.
\newblock Springer-Verlag, Secaucus, NJ, USA, 2005.

\bibitem[Steinwart \& Christmann(2008)Steinwart and
  Christmann]{supportVectorMachines}
Steinwart, I. and Christmann, A.
\newblock \emph{Support vector machines}.
\newblock Springer Science \& Business Media, 2008.

\bibitem[Sugiyama et~al.(2008)Sugiyama, Suzuki, Nakajima, Kashima, von B\"unau,
  and Kawanabe]{Covariate_Shift_jour}
Sugiyama, M., Suzuki, T., Nakajima, S., Kashima, H., von B\"unau, P., and
  Kawanabe, M.
\newblock Direct importance estimation for covariate shift adaptation.
\newblock \emph{Annals of the Institute of Statistical Mathematics},
  60\penalty0 (4):\penalty0 699--746, 2008.

\bibitem[Suzuki et~al.(2009)Suzuki, Sugiyama, and Tanaka]{SuzukiMLMI}
Suzuki, T., Sugiyama, M., and Tanaka, T.
\newblock Mutual information approximation via maximum likelihood estimation of
  density ratio.
\newblock In \emph{Information Theory, 2009. ISIT 2009. IEEE International
  Symposium on}, pp.\  463--467, June 2009.

\bibitem[Tibshirani(1996)]{tibshirani1996lasso}
Tibshirani, R.
\newblock Regression shrinkage and selection via the lasso.
\newblock \emph{Journal of the Royal Statistical Society. Series B
  (Methodological)}, pp.\  267--288, 1996.

\bibitem[Tomioka \& Suzuki(2014)Tomioka and Suzuki]{tomioka2014spectral}
Tomioka, R. and Suzuki, T.
\newblock Spectral norm of random tensors.
\newblock \emph{arXiv preprint arXiv:1407.1870 [math.ST]}, 2014.

\bibitem[Vapnik(1998)]{Vapnik1998}
Vapnik, V.~N.
\newblock \emph{Statistical Learning Theory}.
\newblock Wiley, New York, NY, USA, 1998.

\bibitem[Wainwright(2009)]{WainwrightL1Sharp}
Wainwright, M.~J.
\newblock Sharp thresholds for high-dimensional and noisy sparsity recovery
  using l1-constrained quadratic programming (lasso).
\newblock \emph{IEEE Trans. Inf. Theor.}, 55\penalty0 (5):\penalty0 2183--2202,
  May 2009.

\bibitem[Wainwright \& Jordan(2008)Wainwright and Jordan]{GM_Wainwright}
Wainwright, M.~J. and Jordan, M.~I.
\newblock Graphical models, exponential families, and variational inference.
\newblock \emph{Foundations and Trends{\textregistered} in Machine Learning},
  1\penalty0 (1-2):\penalty0 1--305, 2008.

\bibitem[Wikipedia(2016{\natexlab{a}})]{wiki:Ben_Nelson}
Wikipedia.
\newblock Ben {N}elson --- {W}ikipedia{,} the free encyclopedia,
  2016{\natexlab{a}}.
\newblock URL \url{https://en.wikipedia.org/wiki/Ben_Nelson}.
\newblock [Online; accessed 30-Jan-2016].

\bibitem[Wikipedia(2016{\natexlab{b}})]{wiki:Lincoln-Chafee}
Wikipedia.
\newblock Lincoln {C}hafee --- {W}ikipedia{,} the free encyclopedia,
  2016{\natexlab{b}}.
\newblock URL \url{https://en.wikipedia.org/wiki/Lincoln_Chafee}.
\newblock [Online; accessed 30-Jan-2016].

\bibitem[Yuan \& Lin(2006)Yuan and Lin]{YuanLi2006GroupLasso}
Yuan, M. and Lin, Y.
\newblock Model selection and estimation in regression with grouped variables.
\newblock \emph{Journal of the Royal Statistical Society: Series B (Statistical
  Methodology)}, 68\penalty0 (1):\penalty0 49--67, 2006.

\bibitem[Zhao \& Yu(2006)Zhao and Yu]{zhao2006model}
Zhao, P. and Yu, B.
\newblock On model selection consistency of lasso.
\newblock \emph{The Journal of Machine Learning Research}, 7:\penalty0
  2541--2563, 2006.

\bibitem[Zhao et~al.(2014)Zhao, Cai, and Li]{zhao2014direct}
Zhao, S., Cai, T., and Li, H.
\newblock Direct estimation of differential networks.
\newblock \emph{Biometrika}, 101\penalty0 (2):\penalty0 253--268, 2014.

\end{thebibliography}
\bibliographystyle{icml2016}

\clearpage
\appendix
\onecolumn

\section{Proof of Proposition \ref{prop.pairwise.MV}}
\label{proof.prop.pairwise}
\begin{proof}
	For $X_u \in X1$,
	\begin{align*}
	P(X_u | \backslash X_u) &= \frac{P(X_u, X_{\backslash N(u)} \cap X2| X1 \cup X_{N(u)}\backslash X_u)} {P(X_{\backslash N(u)} \cap X2| X1 \cup X_{N(u)}\backslash X_u)}.
	\end{align*}
	Since $P(X_u | \backslash X_u) = P(X_u | X1 \cup X_{N(u)}\backslash X_u)$ using the Markovian property of PMN, substituting it to the above equation, we have 
	$
	X_u \independent X_{\backslash N(u)} \cap X2| X1 \cup X_{N(u)}\backslash X_u
	$.
	
	$X_v \not \in X_{N(u)}$ means $X_v \in X_{\backslash N(u)} \cap X2$. Using the weak union rule for conditional independence (see e.g., \citep{PGM_Koller}, 2.1.4.3), we obtain $X_u \independent X_v |  \backslash \{X_u, X_v\}$. 
	
	For $X_u \in X2$, the proof is the same. 
\end{proof}

\section{Proof of Theorem \ref{thm.gibbs.to.struct}}
\label{sec.proof.gibbs.to.struct}
\begin{proof}
We define that 
$\mathbf{B}(i)$ is the set of passages contains $X_i$.
Here we only show the proof that Eq. \eqref{eq.markov.1} holds for GPR. Let's denote $\phi_B$ as short for $\phi_B(X_B)$.
\begin{align*}
&P(X_i | X1 \cup X_{N(i)} \backslash X_i)
\\=& \frac{\frac{1}{Z}\int_{X_{\backslash N(i)} \cap X2 } P(X1) P(X_{2}) \prod_{B \in \mathbf{B}(G)}\phi_{B} }{\frac{1}{Z}\int_{X_i}\int_{X_{\backslash N(i)}\cap X2} P(X1) P(X_{2}) \prod_{B \in \mathbf{B}(G)}\phi_{B} }
\\=& \left(\frac{ P(X1)\prod_{B\in \mathbf{B}(i)}\phi_{B} }{\int_{X_i} P(X1)   \prod_{B\in \mathbf{B}(i)}\phi_{B}}\right) \cdot
\left(\frac{\int_{X_{\backslash N(i)}\cap X2}  P(X2) \prod_{B\in \backslash \mathbf{B}(i)}\phi_{B}}{\int_{X_{\backslash N(i)}\cap X2}  P(X2)\prod_{B\in \backslash \mathbf{B}(i)}\phi_{B}}\right)
\\=& \frac{
	P(X1)\prod_{B\in \mathbf{B}(i)}\phi_{B} }{\int_{X_i}   P(X1)
	\prod_{B\in \mathbf{B}(i)}\phi_{B} }
	\\=& \frac{P(X1)
		\prod_{B\in \mathbf{B}(i)} \phi_{B} }{\int_{X_i}   P(X1)
		\prod_{B\in \mathbf{B}(i)} \phi_{B} }\cdot \frac{\frac{1}{Z} P(X2)\prod_{B\in \backslash \mathbf{B}(i)} \phi_{B}}{ \frac{1}{Z} P(X2)\prod_{B\in \backslash \mathbf{B}(i)} \phi_{B}}
		\\=& P(X_i | \backslash  X_{i}),
		\end{align*}
		from which, we obtain the desired equality. Note that we used the fact that $X_{\mathbf{B}(i)} \cap (X_{\backslash N(i)} \cap X2) = \emptyset$ from the second to the third and fourth line.
		\end{proof}
\section{Proof of Theorem \ref{thm.struct.to.gibbs}}
\label{sec.proof.struct.to.gibbs}
	\begin{proof}
		This proof is constructive. Let's clarify some notations used in this proof. 
		Lower-case bold letter $\boldsymbol{a}$ is a vector-realization of a set of random variables $A$. $P(\boldsymbol{a}_K, \boldc)$ means the probability of a realization where elements appearing on positions indexed by subgraph $K$ are allowed to take random values, while other elements are fixed to value $c \in \mathrm{dom}(X)$. Note $K$ might be $\emptyset$. We denote $P1(X)$ as the equivalency of marginal  $P(X1)$. 
		
		First we define the following potential function:
		\begin{align*}
		\phi_S(X_S=\boldx_S) = \prod_{Z \subseteq S} \Delta_{Z}(X_Z = \boldx_Z)^{(-1)^{|S|-|Z|}},
		\end{align*}
		where $S$ is a subset of $G$,  and
		\begin{align}
		\label{eq.delta}
		\Delta_{Z}(\boldx_Z) = 
		&\begin{cases}
		\frac{P(\boldx_Z, \boldc)}{P1(\boldx_{Z}, \boldc)P2(\boldx_{Z}, \boldc) }
		, &\exists{B \in \mathbf{B}(G), B\subseteq Z},  \\
		1&\text{otherwise},
		\end{cases}
		\end{align}
		
		First we show by construction, the multiplication of all potential functions over all subgraph structures,
		i.e., $\prod_{S\subseteq G} \phi_S$ will actually give us the \textbf{PR}.
		
		Due to the \emph{inclusion-exclusion} principle  (see, e.g.\citet{PGM_Koller}, 4.4.2.1), it can be shown that 
		\begin{align*}
		\prod_{S\subseteq G} \phi_S(X_S=\boldx_S) = \Delta_G(\boldx).
		\end{align*}
		If the graph $G$ contains any passage, then by definition $\Delta_G(\boldx) = \frac{P(\boldx)}{P1(\boldx) P2(\boldx)}$, which is exactly the PR.
		However, if $G$ does not include any passage, meaning $X_{1}$ is completely independent of $X_{2}$, then $\Delta_G(\boldx) = 1$ by definition, which is the exact value that a PR would take in such case.  
		
		Second, we show this construction under PMN condition is actually a \textbf{GPR}. Specifically, we show if $S$ is not a passage, then $\phi_S(X_S=\boldx_S) = 1$, i.e. its potential function is nullified. 
		
		Obviously, for a ``one-sided $S$'', $X_S \cap X1 = \emptyset$ or $X_S \cap X2 = \emptyset$, by definition, $\phi_S = 1$. 
		
		Otherwise, if $S$ are ``two-sided'' but itself is not a passage, we should be able to find two nodes, indexed by $X_u \in X1 \cap X_S$ and $X_v \in X2  \cap X_S $, that are not connected by an edge. We may write the potential function for a subgraph $S$ as
		\begin{align*}
		&\phi_S(X_S=\boldx_S) = \prod_{W\subseteq S\backslash \{u,v\}}  \left(\frac{\Delta_w(\boldx_W)\Delta_{W\cup \{u,v\}}(\boldx_{W\cup \{u,v\}})}{\Delta_{W\cup\{u\}}(\boldx_{W \cup \{u\}})\Delta_{W\cup \{v\}}(\boldx_{W\cup \{v\}})}\right)^{*},
		\end{align*}
		where $*$ means we do not care the exact power which can be either -1 or 1, and 
		\begin{align}
		\frac{\Delta_W(\boldx_W)\Delta_{W\cup \{u,v\}}(\boldx_W)}{\Delta_{W\cup\{u\}}(\boldx_{W \cup \{u\}})\Delta_{W\cup \{v\}}(\boldx_{W\cup \{v\}})} = 
		\frac{P_W  P_{W\cup \{u,v\}}}{P_{W\cup \{u\}}P_{W\cup \{v\}}}\cdot
		\frac{P2_{W\cup \{v\}}P2_{W}  P1_{W\cup \{u\}} P1_{W} }{P1_{W} P2_{W} P1_{W\cup \{u\}} P2_{W\cup \{v\}}}	\label{eq.passage},
		\end{align}
		where we have simplified the notation $P(\boldx_A,\boldc)$ as $P_A$.
		The second factor in RHS, \eqref{eq.passage} is apparently 1. For the first factor in RHS, \eqref{eq.passage}, we may divide both the numerator and denominator by $P_W \cdot P_W$. Then it yields 
		$
		\frac{P(x_u, x_v
			|\boldx_W,\boldc)}{P(x_u | \boldx_W,\boldc)P(x_v | \boldx_W,\boldc)}
		$
		which equals to one if and only if $X_u \independent X_v | \backslash \{X_u, X_v\}$. This is guaranteed by PMN condition and Proposition \ref{prop.pairwise.MV}.  
		%	The fact that $a$ and $b$ are not connected in $G$ is implied by the Markov property, and the equality therefore holds under conditional independence between $X_u$ and $X_v$.
	\end{proof}

\section{Proof of Theorem \ref{corol.bounded.ratio}} 
\label{proof.main.them}
Since the PR is a density ratio between the joint density $p(\boldx_1, \boldx_2)$ and the product of two marginals $p(\boldx_1)p(\boldx_2)$, and the objective \eqref{eq.obj.alter} is derived from the same sparsity inducing KLIEP criteria as it was discussed in \citet{Liu2015Change,arxiv_onMNChange}.
The proof of Theorem \ref{corol.bounded.ratio} follows the primal-dual witness procedure \citep{WainwrightL1Sharp}.

First, the Assumptions \ref{assum.depen}, \ref{assum.incoherence} and \ref{assum.smooth} we have made in Section \ref{sec.consistency} is essentially the same as those were imposed in Section 3.2 in \citet{arxiv_onMNChange} (The Hessian of the negative log-likelihood is the sample Fisher information matrix). Then the proof follows the steps established in Section 4, \citet{arxiv_onMNChange}. However, the only thing we need to verify is that $\max_t \|\nabla_{\boldtheta_t}\ell(\boldtheta^*)\|$ is upper-bounded with high probability as $n \rightarrow \infty$. We formally state this in the following lemma:
\begin{lem}
	\label{lemma5.maintext}
	If $\lambda_n \ge \frac{24(2-\alpha)}{\alpha}\cdot\sqrt{\frac{c\log (m^2+m)/2}{n}}$, then 
	\[	P\left(\max_t \|\nabla_{\boldtheta_t}\ell(\boldtheta^*)\| \ge \frac{\alpha \lambda_n}{4(2-\alpha)} \right)\le  3\exp\left(-c''n \right),\]
	where $c$ and $c''$ are some constants.
\end{lem}
\begin{proof}
	For conveniences, let's denote the \emph{approximated} PR model ${\exp\left(\sum_{ u\le v} \boldtheta_{u,v }^\top \boldpsi(\boldx_{u,v})\right)}/{\hat{N}(\boldtheta)}$ as $\hat{g}(\boldx;\boldtheta)$. Since 
$
	\hat{g}(\boldx;\boldtheta) = \frac{N(\boldtheta)}{\hat{N}(\boldtheta)} g(\boldx;\boldtheta),
$
and	$\frac{\hat{N}(\boldtheta)}{N(\boldtheta)} = \frac{1}{{n \choose 2}} \sum_{j\neq k} g(\boldx^{[j,k]}; \boldtheta)$ is always bounded by $[C_\mathrm{min},C_\mathrm{max}]$, we can see $\hat{g}(\boldx;\boldtheta)$ is also bounded. For simplicity, we write 
\begin{align*}
		0 < C'_\mathrm{min} \le \hat{g}(\boldx; \boldtheta) \le C'_\mathrm{max} < \infty.
\end{align*}

	We have
	\[
	\nabla_{\boldtheta_t}\ell(\boldtheta^*) = 
	-\left[\frac{1}{n}\sum_{i=1}^{n} \boldf_t(\boldx^{(i)})\right]+
	\left[ \frac{1}{{n \choose 2}}\sum_{j\le k} \hat{g}(\boldx^{[j,k]};\boldtheta^*) \boldf_t(\boldx^{[j,k]})\right].
	\]
	
	%		and now we show this leads to the upper-boundedness of $\|\nabla_{\boldtheta_t}\ell(\boldtheta^*)\|$.
	%		so now we show by using the boundedness of $\|\boldf_t(\boldx)\|$ implied from Assumption \ref{assum.bounded.ratio}, this is upper-bounded in high probability.
	
	First we show that $\|\nabla_{\boldtheta_t}\ell(\boldtheta^*)\|$ can be upper-bounded as:
	\begin{align*}
	\|\nabla_{\boldtheta_t}\ell(\boldtheta^*) \| &\le \underbrace{\left\| -\frac{1}{n}\sum_{i=1}^{n} \boldf_t(\boldx^{(i)}) +\mathbb{E}_{p}[\boldf_t(\boldx)]\right\|}_{a_{n}}\\
	&+\underbrace{\left\| \frac{1}{{n \choose 2}}\sum_{j\neq k} \hat{g}(\boldx^{[j,k]};\boldtheta^*) \boldf_t(\boldx^{[j,k]}) - \frac{1}{{n \choose 2}}\sum_{j\neq k} g(\boldx^{[i,k]};\boldtheta^*) \boldf_t(\boldx^{[j,k]}) \right\|}_{b_{n}}\\
	&+\underbrace{\left\| \frac{1}{{n\choose 2}}\sum_{j\neq k} g(\boldx^{[j,k]};\boldtheta^*) \boldf_t(\boldx^{[j,k]}) - \mathbb{E}_{p1,p2}\left[ g(\boldx;\boldtheta^*)\boldf_t(\boldx) \right]\right\|}_{\|\boldw_n\|},
	\end{align*}
	We now need Hoeffding inequality \cite{hoeffdingInequality} for bounded-norm vector random variables  which has appeared in previous literatures such as \cite{supportVectorMachines}: For a set of bounded zero-mean vector-valued random variable $ \{\boldy_i\}^n_{i=1}, \|\boldy\|\le c$, we have  
	\begin{align*}
	P(\left\|\sum_{i=1}^n \boldy_i \right\| \ge n\epsilon)\le \exp\left(\frac{-n\epsilon^2}{2c^2}\right),
	\end{align*}
	for all 
	$\epsilon\ge \frac{2c}{\sqrt{n}}.$
	Now it is easy to see
	\begin{align}
	\label{eq.acnq}
	P(a_{n} \ge \epsilon) \le \exp\left(- \frac{2n\epsilon^2}{C'^2_\mathrm{\boldf_t,\mathrm{max}}} \right)
	\end{align}
	as long as 
	\begin{align}
	\label{lem2.hoeffding}
	\epsilon \ge \frac{C'_\mathrm{\boldf_t,\mathrm{max}}}{2\sqrt{n}}.
	\end{align}
	As to $b_{n}$, it can be upper-bounded by 
	\begin{align*}
	b_{n} &= \left\| \frac{1}{{n \choose 2}}\sum_{j\neq k} g(\boldx^{[j,k]};\boldtheta^*) \boldf_t(\boldx^{(i)}) - \frac{1}{{n \choose 2}}\sum_{j\neq k} \hat{g}(\boldx^{[j,k]};\boldtheta^*) \boldf_t(\boldx^{[j,k]}) \right\|\\
	&=\left\|\frac{\hat{N}(\boldtheta^*)}{N(\boldtheta^*)} \frac{1}{{n \choose 2}}\sum_{j\neq k} \hat{g}(\boldx^{[j,k]};\boldtheta^*) \boldf_t(\boldx^{[j,k]}) - \frac{1}{{n \choose 2}}\sum_{j\neq k} \hat{g}(\boldx^{[j,k]};\boldtheta^*) \boldf_t(\boldx^{[j,k]}) \right\|\\
	&\le\left\|\frac{1}{{n \choose 2}}\sum_{j\neq k} \hat{g}(\boldx^{[j,k]};\boldtheta^*) \boldf_t(\boldx^{[j,k]})\right\|\cdot \left\|\frac{\hat{N}(\boldtheta^*)}{N(\boldtheta^*)} - 1 \right\|\\
	&\le C'_\mathrm{max} C'_\mathrm{\boldf_t,\mathrm{max}} \left|\frac{1}{{n \choose 2}}\sum_{j\neq k} g(\boldx^{[j,k]};\boldtheta^*) - 1\right|,
	\end{align*}
	and due to Hoeffding inequality of the U-statistics (see \citep{hoeffdingInequality}, 5b) we may obtain:
	\begin{align}
	\label{eq.bnq}
	P(b_{n} > \epsilon) < 2\exp\left(- \frac{2n\epsilon^2}{C^2_\mathrm{max} C'^2_\mathrm{max} C'^2_\mathrm{\boldf_t,\mathrm{max}}} \right).
	\end{align}
	As to $\boldw_n$, we first bound its $i$-th element $w_{i,n}$ using Hoeffding inequality for U-statistics,
	\begin{align*}
	P(|w_{i,n}| \ge \epsilon) \le 2\exp\left(- \frac{2nb\epsilon^2}{C^2_\mathrm{max} C^2_\mathrm{\boldf_t,\mathrm{max}}} \right),
	\end{align*}
	thus by using the union bound, we have
	\begin{align*}
	P(\|\boldw_n\|_{\infty} \ge \epsilon) \le 2b\exp\left(- \frac{2nb\epsilon^2}{C^2_\mathrm{max} C^2_\mathrm{\boldf_t,\mathrm{max}}} \right),
	\end{align*}
	and since $\|\boldw_n\| \le \sqrt{b}\|\boldw_n\|_\infty$, we have
	\begin{align}
	\label{eq.wn}
	P( \|\boldw_n\| \ge \epsilon) \le P(\sqrt{b}\|\boldw_n\|_{\infty} \ge \epsilon) \le 2b\exp\left(- \frac{2n\epsilon^2}{C^2_\mathrm{max} C^2_\mathrm{\boldf_t,\mathrm{max}}} \right).
	\end{align}
	Therefore, combining \eqref{eq.acnq}, \eqref{eq.bnq} and \eqref{eq.wn}:
	\begin{align*}
	P(\|\nabla_{\boldtheta_t}\ell(\boldtheta^*)\| \ge 3\epsilon) \le P( a_{n}+b_{n}+c_{n}\ge 3\epsilon) & \le c''\exp\left(- \frac{n\epsilon^2}{c'} \right),
	\end{align*}
	where $c'$ is a constant defined as $c' = \max \left(\frac{1}{2}C^2_\mathrm{max}C^2_\mathrm{max} C'^2_\mathrm{\boldf_t,\mathrm{max}}, \frac{1}{2}C^2_\mathrm{max} C^2_\mathrm{\boldf_t,\mathrm{max}}, \frac{1}{2} C'^2_\mathrm{\boldf_t,\mathrm{max}}\right)$, and $c'' = 2b+3$, given $\epsilon \ge \frac{2C'_\mathrm{\boldf_t,\mathrm{max}}}{\sqrt{n}}$.
	Applying the union-bound for all $t\in S\cup S^c$,
	\begin{align*}
	P(\max_{t \in S\cup S^c} \|\nabla_{\boldtheta_t}\ell(\boldtheta^*)\| \ge 3\epsilon) \le  \frac{c''(m^2+m)}{2}\exp\left(- \frac{n\epsilon^2}{c'} \right),
	\end{align*}
	\begin{align*}
	P\left(\max_{t \in S\cup S^c} \|\nabla_{\boldtheta_t}\ell(\boldtheta^*)\| \ge \frac{\alpha \lambda_{n}}{4(2-\alpha)}\right) \le  \frac{c''(m^2+m)}{2}\exp\left(- \left(\frac{\alpha\lambda_{n}}{12(2-\alpha)}\right)^2 \frac{n}{c'} \right),
	\end{align*}
	and when 
	$\lambda_{n} \ge \frac{24(2-\alpha)}{\alpha}\sqrt{\frac{c'\log (m^2+m)/2}{n}}$,
	\begin{align*}
	P\left(\max_{t \in S\cup S^c} \|\nabla_{\boldtheta_t}\ell(\boldtheta^*)\| \ge \frac{\alpha \lambda_{n}}{4(2-\alpha)} \right)\le  c''\exp\left(-c'''n \right),
	\end{align*}
	where $c'''$ is a constant.
	Assume that $\log \frac{m^2+m}{2} > 1$ and we set $\lambda_{n}$ as 
	\[
	\lambda_n \ge \frac{24(2-\alpha)}{\alpha}\sqrt{\frac{(c'+C^2_\mathrm{\boldf_t,\mathrm{max}})\log (m^2+m)/2}{n}},
	\]
	then \eqref{lem2.hoeffding}, the condition of using vector Hoeffding-inequality is satisfied.
\end{proof}

Given Lemma \ref{lemma5.maintext}, we may obtain other technical results, such as the estimation error bound, using the same proof as it was demonstrated in Section 4, \citet{arxiv_onMNChange}.

\section{Experimental Settings}
\label{exp.setting}
We measure the performance of three methods using True Postive Rate (TPR) and True Negative Rate (TNR) that are used in \citet{zhao2014direct}. The TPR and TFR are defined as: 
\begin{align*}
\mathrm{TPR} = \frac{\sum_{t'\in S} \delta(\hat{\boldtheta}_{
		t'} \neq \boldzero)}{\sum_{t'\in S} \delta(\boldtheta^*_{t'} \neq \boldzero)}, 	~~\mathrm{TNR} = \frac{\sum_{t'' \in S^c} \delta(\hat{\boldtheta}_{t''} = \boldzero)}{\sum_{t''\in S^c} \delta(\boldtheta^*_{t''} = \boldzero)},
\end{align*}
where $\delta$ is the indicator function.

The differential learning method \citep{zhao2014direct} used in Section \ref{sec.syn} learns the difference between two precision matrices. In our setting, if one can learn the difference between the precision matrices of $p(\boldx)$ and $p(\boldx_1)p(\boldx_2)$, one can figure out all edges that go across two groups ($\boldx_1$ and $\boldx_2$). 

This method requires sample covariance matrices of $p(\boldx)$ and $p(\boldx_1)p(\boldx_2)$ respectively. The sample covariance of $p(\boldx)$ is easy to compute given joint samples. However, to obtain the sample covariance of $p(\boldx_1)p(\boldx_2)$, we would again need the U-statistics \citep{hoeffdingInequality} introduced in line Section \ref{sec.algorithm}. We may approximate the $u,v$-th element of the covariance matrix of $p(\boldx_1)p(\boldx_2)$ using the formula:
$\Sigma_{u,v} = \frac{1}{{n \choose 2}} \sum_{j \neq k} x_v^{[j,k]}x_u^{[j,k]},$
assuming the joint distribution has zero mean.

\section{Illustration of Sequence Matching}
\begin{figure*}[t]
\centering
\includegraphics[width = .95\textwidth]{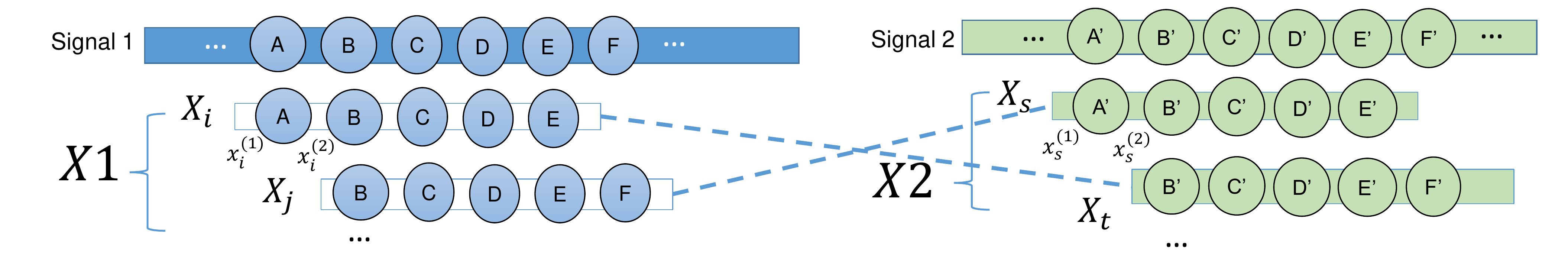}
\caption{The illustration of sequence matching problem formulation.}
\label{fig.illus.seq.matching}
\end{figure*}
We plot the illustrations of our sequence matching  problem formulation from two sequences in Figure \ref{fig.illus.seq.matching}.
\end{document}